\documentclass{article} 
\usepackage{iclr2027_conference,times}

\usepackage{amsmath,amsfonts,bm}

\def\eqref#1{equation~\ref{#1}}

\def\1{\bm{1}}

\DeclareMathAlphabet{\mathsfit}{\encodingdefault}{\sfdefault}{m}{sl}
\SetMathAlphabet{\mathsfit}{bold}{\encodingdefault}{\sfdefault}{bx}{n}

\usepackage{url}
\usepackage[utf8]{inputenc} 
\usepackage[T1]{fontenc}    
\usepackage[colorlinks=true, citecolor=blue]{hyperref}
\usepackage{booktabs}       
\usepackage{amsfonts}       
\usepackage{nicefrac}       
\usepackage{microtype}      
\usepackage{xcolor}         

\usepackage{rotating}
\usepackage{graphicx}
\usepackage{subcaption}
\usepackage{amsmath}
\usepackage{algorithm}
\usepackage{algorithmic}
\usepackage{multirow}
\usepackage{adjustbox}
\usepackage{amsthm}
\usepackage{wrapfig}
\usepackage{enumitem}
\usepackage{float}
\usepackage{amssymb}
\usepackage{colortbl}
\usepackage{makecell}
\usepackage{pifont}
\usepackage{array}
\usepackage{caption}
\usepackage{tabularx}
\usepackage{makecell}

\definecolor{Gray}{gray}{0.93}
\definecolor{linkpink}{HTML}{FF2D8B}

\newtheorem{lemma}{Lemma}
\newtheorem{theorem}{Theorem}
\newtheorem{assumption}{Assumption}

\title{Vocabulary-Incomplete Unsupervised Vision-Language Model Adaptation}

\author{
	Si Qin$^{1}$, Jiawei Tang$^{1}$, Yaxin Hou$^{1}$, Yan Chen$^{2}$, Yutong Xie$^{2}$, Chen Li$^{3}$, Yuheng Jia$^{1,4}$\thanks{Corresponding author.}\\
	$^1$School of Computer Science and Engineering, Southeast University, Nanjing 210096, China \\
	$^2$College of Software Engineering, Southeast University, Nanjing 210096, China \\
	$^3$Department of Biochemistry and Molecular
	Biology, Monash University, Melbourne, \\VIC 3800, Australia \\
	$^4$Key Laboratory of New Generation Artificial Intelligence Technology and Its \\
	Interdisciplinary Applications (Southeast University), Ministry of Education, China \\
	\texttt{\{siqin,230259148,yaxin,220256447,220256506\}@seu.edu.cn,}\\
	\texttt{chen.li@monash.edu, yhjia@seu.edu.cn} \\
}

\iclrfinalcopy 
\begin{document}

\maketitle

\lhead{Under review as a conference paper at ICLR 2027}

\begin{abstract}
	In certain real-world scenarios, obtaining labeled data is excessively expensive, motivating increasing interest in adapting models without annotations. Vision-language models (VLMs) provide pseudo-labels for unlabeled images through image-text similarity, enabling adaptation without manual labeling. The previous unsupervised vision-language model adaption methods assume access to an adaptation vocabulary containing all class names. However, a complete vocabulary is difficult to construct in real-world scenarios, as unlabeled data may contain classes that are overlooked or fall outside the current task scope. Thus, we propose \textbf{V}ocabulary-\textbf{I}ncomplete \textbf{U}nsupervised \textbf{V}ision-language model \textbf{A}daptation (VI-UVA), where only a subset of class names are available during adaptation. The vocabulary incompleteness induces an adaptation bias toward provided classes, which degrades omitted-class recognition. To address this issue, we propose \textbf{L}abel \textbf{S}pace \textbf{R}econstruction (LSR), which reconstructs a latent label space from the unlabeled data and incorporates it into prompt tuning, thereby extending adaptation beyond the incomplete vocabulary and mitigating the adaptation bias. Furthermore, we derive an omitted-class semantic discrepancy bound for LSR and identify the key factors affecting the bound. Extensive experiments across eight benchmarks demonstrate that LSR outperforms existing state-of-the-art methods.
\end{abstract}

\section{Introduction}
Vision-language models (VLMs) exhibit strong open-vocabulary generalization capabilities but often require fine-tuning to further improve performance on specific downstream tasks~\citep{CLIC,Meta-CLIP2}. Nevertheless, acquiring labeled data for fine-tuning can be excessively expensive in certain real-world scenarios~\citep{CAP,AiR}, driving increasing interest in adapting VLMs without manual annotations. \textbf{U}nsupervised \textbf{V}ision-language model \textbf{A}daptation (\textbf{UVA}) \citep{UPL2022} aims to improve the downstream performance of pre-trained VLMs with unlabeled data. Existing UVA methods \citep{FPL_GRIP,CPL} typically construct textual prototypes from the adaptation vocabulary and assign pseudo-labels to unlabeled data based on image-text similarity. High-confidence pseudo-labeled samples are then selected to fine-tune the VLM for the target task. While effective in UVA, these methods rely on the assumption that the adaptation vocabulary contains all class names present in the unlabeled data.

However, this assumption may not hold in real-world scenarios, where unlabeled data contains classes beyond the scope of the current adaptation task. For example, when researchers adapt a VLM to monitor target species, the adaptation vocabulary may omit some species present in the wildlife images because they fall outside the current task scope or are overlooked during task specification~\citep{zhu2022class,wildclip,Open-Insect}. As monitoring expands beyond target species, researchers need to apply the adapted model to recognize other species. The species names previously omitted can be added to the evaluation vocabulary as textual queries~\citep{wildclip,BioCLIP}. Motivated by this scenario, we propose \textbf{V}ocabulary-\textbf{I}ncomplete \textbf{U}nsupervised \textbf{V}ision-language model \textbf{A}daptation (\textbf{VI-UVA}). In this setting, only a subset of class names (i.e., provided-class names) is available during adaptation, while the complete vocabulary is only available at the evaluation stage. Fig.~\ref{fig:introductiona} shows the difference between UVA and VI-UVA.

\begin{figure*}[ht]
	\centering
	\begin{subfigure}[t]{0.63\textwidth}
		\centering
		\includegraphics[width=\linewidth]{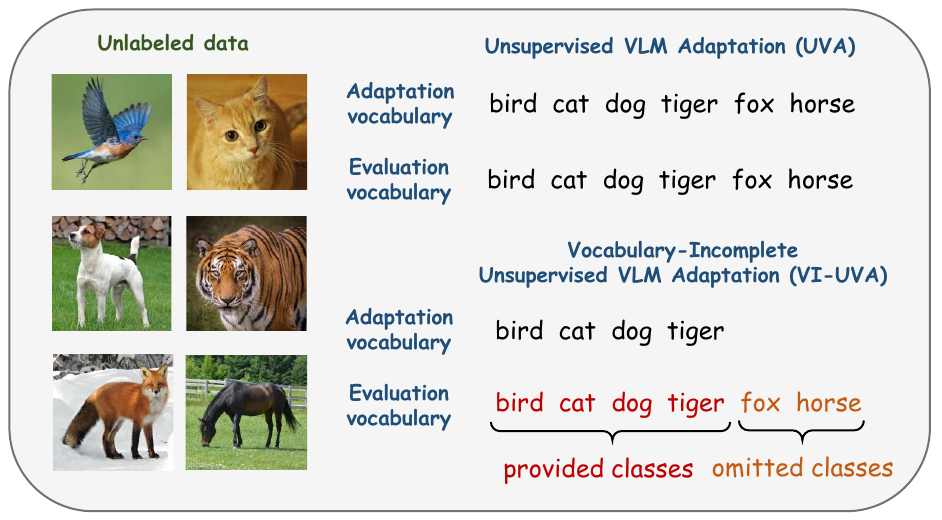}
		\caption{}
		\label{fig:introductiona}
	\end{subfigure}
	\hfill
	\begin{subfigure}[t]{0.36\textwidth}
		\centering
		\includegraphics[width=\linewidth]{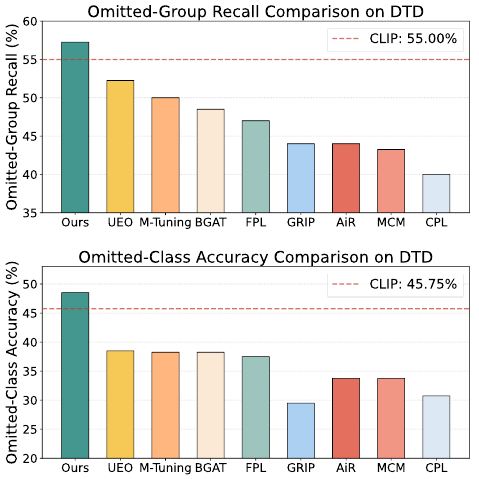}
		\caption{}
		\label{fig:introductionb}
	\end{subfigure}
	\vspace{-0.2cm}
	\caption{(a) Comparison between unsupervised VLM adaptation (UVA) and vocabulary-incomplete unsupervised VLM adaptation (VI-UVA). During adaptation, all class names are available for UVA, while only a subset of class names is available for VI-UVA. During evaluation, UVA and VI-UVA are both tested on the complete vocabulary including provided-class names and omitted-class names. (b) Omitted-group recall measures the proportion of omitted-class samples assigned to any omitted class, which indicates how well the model avoids vocabulary overfitting. On the DTD dataset, our method outperforms the best adaptation method by 5.0 and 10.0 percentage points in omitted-group recall and omitted-class accuracy, respectively. }
	\label{fig:introduction}
\end{figure*}

Existing UVA methods rely on pseudo-labeling to adapt VLMs. However, under the proposed setting, pseudo-labeling may assign omitted-class images to provided classes, while training exclusively with the incomplete vocabulary drives the learned prompts toward the provided classes, inducing an \textbf{adaptation bias}. This bias leads to vocabulary overfitting (i.e., the model tends to classify omitted-class samples as provided classes) and degrades omitted-class accuracy (Fig.~\ref{fig:introductionb}). Thus, the main issue is to improve provided-class recognition and preserve omitted-class recognition, even though the information of omitted classes are unavailable during adaptation.

To address this issue, we propose \textbf{L}abel \textbf{S}pace \textbf{R}econstruction (LSR). The main idea is to exploit the latent semantic structures in unlabeled data, which contains images from both provided and omitted classes, offering information beyond the adaptation vocabulary. To extract this semantic information, we train a lightweight \textbf{C}ross-\textbf{M}odal \textbf{S}emantic \textbf{P}rojector (CMSP) that maps visual representations into the textual embedding space. We apply CMSP to all unlabeled images and cluster the generated textual embeddings to obtain semantic anchors. Each anchor is associated with a corresponding pseudo-token sequence, constructing a pseudo-vocabulary that encodes latent semantics. By incorporating the pseudo-vocabulary into prompt tuning, LSR prevents the prompts from being tuned only to the provided classes. This helps mitigate the adaptation bias and preserve omitted-class recognition (Fig.~\ref{fig:introductionb}). Finally, we derive an omitted-class semantic discrepancy bound for LSR, characterizing how projection quality, cross-modal recomposition, and representation approximation jointly govern the bound.

In summary, the main contributions of this work are as follows.

\begin{itemize}[leftmargin=*]
	\item For the first time, we propose \textbf{V}ocabulary-\textbf{I}ncomplete \textbf{U}nsupervised \textbf{V}ision-language \textbf{M}odel \textbf{A}daptation (VI-UVA), a practical yet challenging setting in which only a subset of class names is available during adaptation. The vocabulary incompleteness induces an adaptation bias, which leads to vocabulary overfitting and degrades omitted-class recognition.
	
	\item We propose \textbf{L}abel \textbf{S}pace \textbf{R}econstruction (LSR), which reconstructs a latent label space from the unlabeled data and incorporates it into prompt tuning, thereby extending fine-tuning beyond the adaptation vocabulary and mitigating the adaptation bias.
	
	\item We establish an omitted-class semantic discrepancy bound for LSR, which is jointly determined by projection quality, cross-modal recomposition, and representation approximation.
	
	\item Extensive experiments on eight commonly used benchmarks at different omission ratios demonstrate that our method achieves state-of-the-art performance.
\end{itemize}

\section{Related Work}
\subsection{Vision-Language Models}
Vision-Language Models (VLMs)~\citep{CLIP,ALBEF,BLIP,Florence-2} learn cross-modal representations from large-scale image-text pairs. Among them, CLIP~\citep{CLIP} consists of an image encoder and a text encoder that project images and textual descriptions into a shared embedding space. This alignment enables classification through textual descriptions without additional training on the target task. Specifically, class names are inserted into predefined prompt templates, such as \textit{a photo of a [class]}, and encoded as textual prototypes. Each input image is then assigned to the class whose textual prototype has the highest similarity to its visual representation. Despite their strong zero-shot capabilities, VLMs suffer substantial performance degradation under domain shifts, motivating adaptation to downstream tasks. Existing studies adapt VLMs by optimizing prompts with labeled data~\citep{CoOp,CoCoOp,VLMsSurvey}. In this paper, we focus on unsupervised VLM adaptation, thereby eliminating the need for costly annotations.

\subsection{Unsupervised Vision-Language Model Adaptation}
In certain real-world scenarios, acquiring even a small number of annotations can be excessively expensive~\citep{CPL,CAP,AiR}. This limitation has motivated increasing interest in unsupervised VLM adaptation, where zero-shot predictions from VLMs are used as pseudo-labels for prompt tuning. Representative methods such as FPL and GRIP~\citep{FPL_GRIP} select the top-$k$ highest-confidence samples for each class, with GRIP further refining the pseudo-labels and progressively increasing $k$ during adaptation. Subsequent studies extend this framework by introducing partial labels~\citep{CPL}, mitigating prediction bias through calibrated sample selection~\citep{CAP}, or incorporating auxiliary supervision from generated images~\citep{AiR}. However, these methods implicitly assume that the adaptation vocabulary is complete. In contrast, we investigate a vocabulary-incomplete setting in which only a subset of class names is available during adaptation.

\subsection{Related Learning Paradigms}
Several learning paradigms relax the closed-set assumption and address different forms of openness. Open-set recognition~\citep{OSR2016} and out-of-distribution detection~\citep{OOD2016} reject samples outside the known classes or training distribution, whereas classes omitted from the adaptation vocabulary in VI-UVA still require fine-grained classification. Generalized category discovery~\citep{GCD2022,GET} classifies known classes and discovers novel classes by clustering unlabeled data, typically with access to labeled data from known classes, while VI-UVA uses only unlabeled data and an incomplete vocabulary to classify all classes. A detailed comparison of VI-UVA with these learning paradigms is shown in Appendix  \ref{app:comparison_with_related_learning_paradigms}. Overall, VI-UVA seeks to perform classification over all classes while relying on weaker supervision, making it more challenging.

\section{Proposed Method}
\textbf{Overview.} In this paper, we propose \textbf{L}abel \textbf{S}pace \textbf{R}econstruction (LSR) for vocabulary-incomplete unsupervised vision-language model adaptation. LSR consists of three stages, including cross-modal semantic projection  (Sec.~\ref{sec:CMSP}), latent label space reconstruction (Sec.~\ref{sec:LLSR}), and reconstruction space adaptation (Sec.~\ref{sec:RSA}). The overview and pseudo-code of LSR are detailed in Fig.~\ref{fig:overview} and Algorithm \ref{alg:LSR}, respectively.

\textbf{Problem Setup.} Let $\mathcal{D} _{\mathrm{U}}=\left\{ x_i  \right\} _{i=1}^{N}$ denote an unlabeled dataset containing $N$ samples and $\mathcal{C}_t = \left\{ c_j \right\} _{j=1}^{C}$ denote the set of all class names. Each sample in $\mathcal{D} _{\mathrm{U}}$ belongs to a class whose name is included in $\mathcal{C}_t$. We partition $\mathcal{C}_t$ into two disjoint subsets,
\begin{equation}
	\mathcal{C}_t=\mathcal{C}_p\cup\mathcal{C}_o,~~~
	\mathcal{C}_p\cap\mathcal{C}_o=\emptyset,
\end{equation}
where $\mathcal{C}_p$ and $\mathcal{C}_o$ denote the sets of provided- and omitted-class names, respectively. Specifically, the class names in $\mathcal{C}_p$ constitute the adaptation vocabulary $\mathcal{V}_p$, whereas those in $\mathcal{C}_o$ are unavailable during adaptation, even though their images are present in $\mathcal{D} _{\mathrm{U}}$. 

Given $\mathcal{D}_{\mathrm{U}}$ and $\mathcal{V}_p$, we aim to adapt the VLM without annotations or the class names in $\mathcal{C}_o$. The adapted model is evaluated on $\mathcal{C}_t$, with the goal of maximizing recognition performance across both provided and omitted classes.

\subsection{Cross-Modal Semantic Projection}
\label{sec:CMSP}
The adaptation vocabulary contains only a subset of class names. Pseudo-labeling with this incomplete vocabulary can assign omitted-class images to provided classes, which introduces biased supervision and encourages the model to overfit the adaptation vocabulary. Motivated by~\cite{Bag-of-regions}, we introduce the \textbf{C}ross-\textbf{M}odal \textbf{S}emantic \textbf{P}rojector (CMSP), a lightweight linear mapping that transfers visual representations into the textual embedding space, providing a basis for reconstructing the latent semantic anchors associated with omitted classes.

Denote the image encoder and text encoder of VLMs as $\mathcal{I} \left( \cdot \right) $ and $\mathcal{T}\left( \cdot \right)$, respectively. We extract a visual embedding $\mathbf{x}_i\in \mathbb{R} ^d$ of each $x_i\in \mathcal{D}_{\mathrm{U}}$ through $\mathcal{I} \left( \cdot \right) $ and extract a textual embedding $\mathbf{c}_j\in \mathbb{R} ^d$ of each $c_j\in \mathcal{V}_p$ through $\mathcal{T}\left( \cdot \right)$. Based on image-text similarity, these textual embeddings are used to generate pseudo-labels for unlabeled images. Following prior work~\citep{FPL_GRIP}, we select the same number of highest-confidence images for each provided class to construct a balanced pseudo-labeled dataset $\mathcal{D} _\mathrm{P}=\left\{ \left( x_m, p_m \right) \right\} _{m=1}^{M}$, where $p_m \in \left\{1,\cdots,|\mathcal{V}_p|\right\}$ denotes the pseudo-label assigned to image $x_m$. Appendix \ref{app:pseudo_labeled_dataset_generation} provides the details of pseudo-labeled dataset generation. $\mathcal{D} _\mathrm{P}$ is then used to train the cross-modal semantic projector.

The cross-modal semantic projector, parameterized by $\phi$, maps the visual embedding $\mathbf{x}_m \in \mathbb{R}^d$ of each pseudo-labeled image $x_m$ into a flattened token-level representation $\mathbf{h}_m \in \mathbb{R}^{Ld}$, where $L$ is the predefined number of pseudo-tokens. We reshape $\mathbf{h}_m$ into a pseudo-token sequence $\mathbf{H}_m \in \mathbb{R} ^{L\times d}$ and feed it into the frozen text encoder $\mathcal{T}\left( \cdot \right)$ to obtain the generated textual embedding $\mathbf{t}_m \in \mathbb{R}^d$. The complete projection process can be expressed as:
\begin{equation}
\label{eq:projection}
	\mathbf{t}_m =\mathcal{T}\left(\mathcal{R}_{L\times d}\left(f_{\phi}(\mathcal{I}(x_m))\right)\right),
\end{equation}
where $\mathcal{R}_{L\times d}$ represents the reshape operation and $f_\phi$ denotes the CMSP.
The training objective of the CMSP comprises a feature regression loss and a classification loss.

\textbf{Feature regression loss.}
To encourage the generated textual embedding $\mathbf{t}_m$ to preserve the category-level semantics, we minimize its feature distance to the corresponding provided-class prototype. Let $\mathbf{c}_{p_m}$ denote the textual embedding of the pseudo-label $p_m$. Given a mini-batch of $B$ pseudo-labeled samples, the feature regression loss is defined as:
\begin{equation}
	\mathcal{L}_{feat}
	=\frac{1}{Bd}\sum_{m=1}^{B}
	\left\|\mathbf{t}_m-\mathbf{c}_{p_m}\right\|_1,
	\label{eq:feature_regression_loss}
\end{equation}
where $d$ is the embedding dimension. This objective directly encourages the generated textual embeddings to align closely with their corresponding class prototypes.

\textbf{Classification loss.}
Feature regression loss alone does not directly optimize the relative similarities between a generated textual embedding and different provided-class prototypes. Therefore, we calculate its similarities to all provided-class prototypes and apply a classification loss:
\begin{equation}
	\mathcal{L}_{cls}
	=-\frac{1}{B}\sum_{m=1}^{B}\log
	\frac{\exp\left(\mathbf{t}_{m}^\top\mathbf{c}_{p_m}\right)}
	{\sum_{j=1}^{|\mathcal{C}_p|}\exp\left(\mathbf{t}_m^\top\mathbf{c}_j\right)},
\end{equation}
where $|\mathcal{C}_p|$ represents the number of provided classes. This loss encourages each generated textual embedding to be more similar to its corresponding class prototype than to the prototypes of other provided classes. Finally, the overall objective for training the CMSP is formulated as:
\begin{equation}
	\mathcal{L}_{\mathrm{CMSP}}=
	\mathcal{L}_{feat} + \alpha\mathcal{L}_{cls},
	\label{eq:CMSP_loss}
\end{equation}
where $\alpha$ is set to $0.1$ in our implementation. During training, the image and text encoders remain frozen, and only the parameters $\phi$ of the CMSP are optimized.

\subsection{Latent Label Space Reconstruction}
\label{sec:LLSR}
By applying the projection process in Eq.~\ref{eq:projection} to all images in $\mathcal{D}_{\mathrm{U}}$, we obtain the generated textual embeddings $\mathbf{G}=\left[ \mathbf{t}_{1},\mathbf{t}_{2},\cdots, \mathbf{t}_N \right] ^{\top}\in \mathbb{R} ^{N\times d}$ derived from both provided- and omitted-class images. Thus, we exploit these embeddings to reconstruct a latent label space.

\textbf{Latent semantic clustering.}
Since the information of omitted classes are unavailable during adaptation, we adopt an over-clustering strategy to allocate additional semantic capacity beyond the provided classes. Specifically, the number of clusters is set to $2|\mathcal{C}_p|$,
which is twice the number of provided classes. We perform K-means clustering on the generated textual embeddings $\mathbf{G}$ and
the resulting candidate latent semantic anchors are denoted by $\mathcal{M} =\left\{ \boldsymbol{\mu }_k \right\} _{k=1}^{2\left| \mathcal{C} _p \right|}$, where $\boldsymbol{\mu }_k \in \mathbb{R} ^d$ represents the $k$-th cluster center. These cluster centers serve as data-driven semantic anchors and reserve additional capability to represent the concepts beyond the adaptation vocabulary. Importantly, they are not required to recover the ground-truth class names or correspond one-to-one with the underlying classes.

Over-clustering may produce unreliable anchors associated with outliers or noise. We assess each candidate anchor by the number of unlabeled images assigned to it and discard those supported by very few images. This filtering procedure aims to reduce the influence of outliers or noise, while retaining the dominant semantic structures present in the unlabeled data. The reconstructed latent label space is consequently denoted as $\mathcal{M}_{ret} = \left\{ \boldsymbol{\mu }_k \right\} _{k=1}^{r}$, where $r$ represents the number of retained anchors. Appendix \ref{app:filtering_procedure_of_unreliable_centers} provides the details of the filtering procedure.

\subsection{Reconstructed Space Adaptation}
\label{sec:RSA}
Since the retained anchors cannot be directly used as text encoder inputs, we associate them with pseudo-token sequences generated by CMSP for prompt tuning.
Recall that CMSP maps the visual representation of each unlabeled image $x_i$ into a pseudo-token sequence $\mathbf{H}_i\in\mathbb{R}^{L\times d}$ and the text encoder generates the corresponding textual embedding $\mathbf{t}_i=\mathcal{T}(\mathbf{H}_i)$. For each retained anchor $\boldsymbol{\mu}_k\in\mathcal{M}_{ret}$, we retrieve its most similar generated textual embedding $\mathbf{t}_{i_{k}^{*}}$, where
\begin{equation}
	i_{k}^{*}=\underset{i\in \left\{1,...,N\right\}}{\mathrm{arg}\max}\frac{\boldsymbol{\mu }_{k}^{\top}\mathbf{t}_i}{\left\| \boldsymbol{\mu }_k \right\| _2\left\| \mathbf{t}_i \right\| _2}.
	\label{eq:retrieval}
\end{equation}
Consequently, the pseudo-token sequence associated with the $k$-th anchor is denoted as $\mathbf{H}_{i_{k}^{*}}$. 
The reconstructed pseudo-vocabulary and corresponding textual representations are denoted by:
\begin{equation}
	\mathcal{V}_{\mathrm{R}}=\left\{\mathbf{H}_{i_{k}^{*}} \right\} _{k=1}^{r},~~
	\mathbf{T}_{\mathrm{R}}=\left\{\mathbf{t}_{i_{k}^{*}} \right\} _{k=1}^{r}, 
\end{equation}
where $r$ represents the number of retained semantic anchors. We exploit $\mathbf{T}_\mathrm{R}$ to construct a balanced pseudo-labeled dataset $\mathcal{D}_{\mathrm{R}}=\left\{ \left( x_q, p_q \right) \right\} _{q=1}^{Q}$ following the same procedure described in the Appendix \ref{app:pseudo_labeled_dataset_generation}, where $Q$ denotes the number of selected images and $p_q\in\left\{1,\ldots,r\right\}$ represents the latent pseudo-label of the image $x_q$. 

Then, we perform textual prompt tuning \citep{CoOp} with $\mathcal{V}_{\mathrm{R}}$ and $\mathcal{D}_{\mathrm{R}}$. Let $\mathbf{V}_{\theta}=\left[\mathbf{v}_1,\cdots,\mathbf{v}_S\right]^{\top}\in\mathbb{R}^{S\times d}$ denote the shared learnable context tokens, where $S$ is the number of shared tokens. For the $k$-th pseudo-token sequence in $\mathcal{V}_{\mathrm{R}}$, its prompt-conditioned textual prototype is computed as:
\begin{equation}
	\mathbf{z}_k=\mathcal{T}\left(\left[\mathbf{V}_{\theta};
	\mathbf{H}_{i_{k}^{*}}\right]\right),
\end{equation}
where $[\cdot ; \cdot]$ denotes token-level concatenation. Given a mini-batch of $B$ pseudo-labeled samples $\left(x_q,p_q\right)$ from $\mathcal{D}_{\mathrm{R}}$, the cross-entropy loss for optimizing $\mathbf{V}_{\theta}$ is defined as:
\begin{equation}
	\mathcal{L}_{\mathrm{R}}
	=-\frac{1}{B}\sum_{q=1}^{B}\log \frac{\exp \left( \mathrm{sim}\left( \mathcal{I} (x_q),\mathbf{z}_{p_q} \right) /\tau \right)}{\sum_{\ell =1}^r{\exp \left( \mathrm{sim}\left( \mathcal{I} (x_q),\mathbf{z}_{\ell} \right) /\tau \right)}},
\end{equation}
where $\tau$ is the temperature parameter and $\mathrm{sim}\left(\cdot,\cdot\right)$ represents cosine similarity. 
This objective updates $\mathbf{V}_{\theta}$ using supervision from the pseudo-vocabulary $\mathcal{V}_\mathrm{R}$, which incorporates semantic structures derived from both provided and omitted classes, thereby helping preserve omitted-class recognition. To further enhance provided-class recognition, we also optimize $\mathbf{V}_{\theta}$ with the balanced pseudo-labeled dataset $\mathcal{D}_{\mathrm{P}}$. For each provided-class name $c_j\in\mathcal{V}_p$, we construct a prompt-conditioned textual prototype as:
\begin{equation}
	\widetilde{\mathbf{c}}_j
	=\mathcal{T}\left(\left[\mathbf{V}_{\theta};\operatorname{Token}(c_j)\right]\right),
\end{equation}
where $\operatorname{Token}(c_j)$ denotes the token embeddings of the class name $c_j$. Given a mini-batch of $B$ pseudo-labeled samples $(x_m,p_m)$ from $\mathcal{D}_{\mathrm{P}}$, the cross-entropy loss is:
\begin{equation}
	\mathcal{L}_{\mathrm{P}}=-\frac{1}{B}\sum_{m=1}^B\log{\frac{\exp \left( \mathrm{sim}\left( \mathcal{I} (x_m),\widetilde{\mathbf{c}}_{p_m} \right) /\tau \right)}{\sum_{\ell =1}^{\left| \mathcal{C} _p \right|}{\exp \left( \mathrm{sim}\left( \mathcal{I} (x_m),\widetilde{\mathbf{c}}_{\ell} \right) /\tau \right)}}}.
\end{equation}
The overall objective for textual prompt tuning is:
\begin{equation}
	\mathcal{L}_{\mathrm{all}}
	=\lambda\mathcal{L}_{\mathrm{R}}+\mathcal{L}_{\mathrm{P}},
	\label{eq:loss_all}
\end{equation}
where $\lambda$ balances supervision from the reconstructed latent label space and the provided classes.
During prompt tuning, the image encoder, text encoder, CMSP, and pseudo-vocabulary $\mathcal{V}_{\mathrm{R}}$ remain frozen, while only the shared context tokens $\mathbf{V}_{\theta}$ are optimized. At inference time, we combine the shared context tokens with each class name in $\mathcal{C}_t$ to recognize both provided and omitted classes.

\subsection{Theoretical Analysis}
We establish an omitted-class semantic discrepancy bound for LSR, identifying how projection quality, cross-modal recomposition,
and representation approximation jointly govern the bound.

\begin{theorem}[Omitted-Class Semantic Discrepancy Bound]
	\label{thm1:semantic_discrepancy}
	Let $\mathcal{F}(\mathbf{x})=\mathcal{T}\left(\mathcal{R}_{L\times d}\left(f_{\phi}(\mathbf{x})\right)\right)$
	denote the complete visual-to-textual mapping. Suppose that $\mathcal{F}$ and its Jacobian are Lipschitz continuous on a convex domain containing the relevant visual
	representations. Then,
	\begin{equation}
		\mathcal{E}^{\mathrm{rec}}_o\leq\frac{1}{2}\left(
		\varepsilon_{\mathrm{proj}}+\mathfrak{b}_o+\mathfrak{r}_o\right)^2, 
		\label{eq:main_theoretical_bound}
	\end{equation}
	where $\mathcal{E}^{\mathrm{rec}}_o$ denotes the average semantic discrepancy between the nearest retrieved textual representations and corresponding omitted-class prototypes, $\varepsilon_{\mathrm{proj}}$ measures the error of cross-modal semantic projection, $\mathfrak{b}_o$ measures the discrepancy
	of recomposing omitted-class visual and textual semantics, and $\mathfrak{r}_o$ measures the approximation error of the generated textual representations to their nearest retrieved textual representations.
\end{theorem}
Theorem~\ref{thm1:semantic_discrepancy} reveals three factors affecting
omitted-class semantic discrepancy. Accurate projection, lower cross-modal recomposition discrepancy and representation approximation error jointly lead to a tighter bound. The complete proof and further analysis of these factors are provided in Appendix~\ref{app:proof}.

\section{Experiments}
To demonstrate the effectiveness of the proposed method, we conduct experiments across eight datasets and compare with ten SOTA methods. In addition, we conduct ablation studies and analytical experiments to explain why our method is effective.

\subsection{Experimental Settings}
\label{sec:experimental_settings}
\textbf{Datasets and evaluations.} We conduct experiments on eight image classification datasets from diverse domains, including DTD (DTD.) ~\citep{DTD}, FGVC-Aircraft (FA.)~\citep{FGVC-Aircraft}, CUB (CUB.)~\citep{CUB}, Flowers102 (Flw.)~\citep{Flowers102}, OxfordPets (Pets.)~\citep{OxfordPets}, StanfordCars (Cars.)~\citep{StanfordCars}, RESISC45 (RES.)~\citep{RESISC45}, and CIFAR100 (CIFAR.)~\citep{cifar100}. For each dataset, we partition the set of all class names $\mathcal{C}_t$ into provided and omitted subsets. We define the \textbf{omission ratio} as $\gamma=|\mathcal{C}_o|/|\mathcal{C}_t|$ and evaluate multiple values of $\gamma$ to simulate different degrees of vocabulary incompleteness. We evaluate model performance using three metrics, including all-class accuracy, omitted-class accuracy, and omitted-group recall. Further details can be found in Appendix \ref{app:datasets_and_evaluation}.

\textbf{Baselines.} We compare our method with ten SOTA methods,
namely, CLIP~\citep{CLIP}, FPL~\citep{FPL_GRIP}, GRIP ~\citep{FPL_GRIP}, CPL~\citep{CPL}, UEO~\citep{UEO}, AiR~\citep{AiR}, CAP~\citep{CAP}, MCM~\citep{MCM}, M-Tuning~\citep{M-Tuning} and BGAT~\citep{BGAT}. Appendix \ref{app:comparison_methods} describes the details about these methods. For a fair comparison, all adaptation methods use the same pre-trained CLIP backbone and employ textual prompt tuning~\citep{CoOp}. Appendix \ref{app:implementation_details} presents the implementation details.

\subsection{Main Results}
Tables~\ref{tab:main_results_gamma_0.2} and~\ref{tab:main_results_gamma_0.3} compare the all-class accuracy of different methods across eight datasets at omission ratios of $\gamma=0.2$ and $\gamma=0.3$, respectively. Each result is reported as the mean and standard deviation over three random seeds for optimization (1, 2, and 3). Our method consistently achieves the highest accuracy on all eight datasets. At $\gamma=0.2$ and $\gamma=0.3$, LSR outperforms the strongest baseline by an average of \textbf{2.47} and \textbf{2.57} percentage points, respectively. Specifically, LSR surpasses the best baseline on RES by \textbf{3.74} percentage points at $\gamma=0.3$. These results demonstrate the effectiveness of LSR across different datasets and omission ratios, suggesting that the reconstructed latent label space complements the restricted supervision provided by the adaptation vocabulary. We provide more main results in Appendix \ref{app:additional_main_results}.

\vspace{-0.2cm}
\begin{table*}[ht]
	\centering
	\caption{Comparison results of all-class accuracy (\%) at $\gamma=0.2$. Values are expressed as $\mathrm{mean}_{\mathrm{std}}$. The highest accuracies are bold.}
	\vspace{-0.2cm}
	\label{tab:main_results_gamma_0.2}
	\small
	\renewcommand{\arraystretch}{1.1}
	\resizebox{1\linewidth}{!}{
		\begin{tabular}{l|ccccccccc}
			\toprule[1.1pt]
			Methods & Flw. & FA. & CUB. & DTD. & Pets. & Cars. & RES. & CIFAR. & Avg. \\
			\cmidrule(lr){1-1} \cmidrule(lr){2-10}
			CLIP & $63.63_{0.00}$ & $17.64_{0.00}$ & $51.77_{0.00}$ & $43.19_{0.00}$ & $84.51_{0.00}$ & $58.93_{0.00}$ & $54.48_{0.00}$ & $64.19_{0.00}$ & $54.79_{0.00}$ \\
			FPL & $65.99_{0.70}$ & $17.11_{0.46}$ & $51.09_{0.58}$ & $44.29_{0.67}$ & $87.95_{0.61}$ & $57.09_{0.71}$ & $58.08_{0.68}$ & $65.05_{0.55}$ & $55.83_{0.62}$ \\
			GRIP & $60.83_{0.25}$ & $14.68_{0.37}$ & $42.21_{0.72}$ & $44.24_{2.00}$ & $79.87_{0.62}$ & $50.53_{0.35}$ & $62.57_{1.38}$ & $65.02_{1.17}$ & $52.49_{0.86}$ \\
			CPL & $63.50_{1.90}$ & $14.97_{0.65}$ & $46.82_{0.71}$ & $45.53_{0.33}$ & $77.56_{1.10}$ & $47.17_{0.57}$ & $64.06_{0.92}$ & $62.17_{0.53}$ & $52.72_{0.84}$ \\
			CAP & $65.37_{1.60}$ & $15.74_{0.91}$ & $49.18_{0.75}$ & $45.85_{2.25}$ & $80.43_{4.37}$ & $53.92_{0.24}$ & $65.57_{1.32}$ & $65.25_{0.04}$ & $55.16_{1.44}$ \\
			AiR & $60.61_{1.13}$ & $14.98_{0.11}$ & $41.62_{0.72}$ & $45.87_{1.42}$ & $78.88_{2.17}$ & $49.84_{0.11}$ & $63.72_{1.02}$ & $65.33_{0.66}$ & $52.61_{0.92}$ \\
			MCM & $61.81_{0.35}$ & $16.19_{1.23}$ & $45.53_{1.14}$ & $39.73_{0.97}$ & $79.47_{1.74}$ & $51.53_{0.32}$ & $53.11_{0.65}$ & $63.46_{0.45}$ & $51.35_{0.86}$ \\
			UEO & $61.74_{0.92}$ & $8.29_{2.95}$ & $44.52_{0.18}$ & $42.86_{0.20}$ & $82.73_{4.11}$ & $52.39_{1.24}$ & $64.94_{1.13}$ & $58.79_{0.97}$ & $52.03_{1.46}$ \\
			M-Tuning & $66.75_{0.29}$ & $17.31_{0.31}$ & $50.79_{0.23}$ & $44.70_{0.24}$ & $88.06_{0.55}$ & $57.26_{0.12}$ & $58.11_{0.42}$ & $65.41_{0.33}$ & $56.05_{0.31}$ \\
			BGAT & $58.72_{1.41}$ & $18.22_{0.77}$ & $48.24_{0.70}$ & $40.05_{0.16}$ & $82.60_{0.65}$ & $55.87_{0.34}$ & $54.47_{0.42}$ & $64.54_{0.07}$ & $52.84_{0.57}$ \\
			\cellcolor{Gray}LSR (Ours) & \cellcolor{Gray}$\textbf{66.89}_{0.60}$ & \cellcolor{Gray}$\textbf{19.34}_{0.41}$ & \cellcolor{Gray}$\textbf{53.27}_{0.22}$ &
			\cellcolor{Gray}$\textbf{47.02}_{0.54}$ & \cellcolor{Gray}$\textbf{88.49}_{0.60}$ & \cellcolor{Gray}$\textbf{60.15}_{0.34}$ &
			\cellcolor{Gray}$\textbf{65.82}_{0.35}$ & \cellcolor{Gray}$\textbf{67.14}_{0.53}$ & \cellcolor{Gray}$\textbf{58.52}_{0.45}$ \\
			\bottomrule[1.1pt]
	\end{tabular}}
\end{table*}
\vspace{-0.2cm}

\begin{table*}[ht]
	\centering
	\caption{Comparison results of all-class accuracy (\%) at $\gamma=0.3$. Values are expressed as $\mathrm{mean}_{\mathrm{std}}$. The highest accuracies are bold.}
	\vspace{-0.2cm}
	\label{tab:main_results_gamma_0.3}
	\small
	\renewcommand{\arraystretch}{1.1}
	\resizebox{1\linewidth}{!}{
		\begin{tabular}{l|ccccccccc}
			\toprule[1.1pt]
			Methods & Flw. & FA. & CUB. & DTD. & Pets. & Cars. & RES. & CIFAR. & Avg. \\
			\cmidrule(lr){1-1} \cmidrule(lr){2-10}
			CLIP & $63.63_{0.00}$ & $17.64_{0.00}$ & $51.77_{0.00}$ & $43.19_{0.00}$ & $84.51_{0.00}$ & $58.93_{0.00}$ & $54.48_{0.00}$ & $64.19_{0.00}$ & $54.79_{0.00}$ \\
			FPL & $63.29_{0.62}$ & $17.23_{0.10}$ & $50.70_{0.52}$ & $43.14_{0.56}$ & $86.80_{0.36}$ & $55.52_{0.91}$ & $57.14_{1.48}$ & $64.88_{0.44}$ & $54.84_{0.62}$ \\
			GRIP & $53.58_{2.79}$ & $14.65_{0.41}$ & $40.66_{1.31}$ & $42.29_{0.74}$ & $75.66_{1.11}$ & $49.50_{0.87}$ & $59.36_{1.73}$ & $63.06_{1.07}$ & $49.85_{1.25}$ \\
			CPL & $56.59_{0.53}$ & $15.12_{0.80}$ & $45.10_{0.51}$ & $41.72_{1.54}$ & $74.58_{1.75}$ & $45.02_{0.48}$ & $59.32_{1.15}$ & $61.21_{1.48}$ & $49.83_{1.03}$ \\
			CAP & $56.23_{0.51}$ & $14.88_{0.98}$ & $45.64_{0.38}$ & $43.09_{1.03}$ & $81.30_{3.24}$ & $53.64_{0.74}$ & $61.52_{0.36}$ & $64.80_{0.33}$ & $52.64_{0.95}$ \\
			AiR & $57.07_{1.61}$ & $14.11_{1.35}$ & $40.32_{1.02}$ & $42.52_{1.45}$ & $81.78_{1.19}$ & $50.18_{1.30}$ & $59.70_{0.35}$ & $62.69_{0.46}$ & $51.05_{1.09}$ \\
			MCM & $58.76_{0.48}$ & $14.44_{2.23}$ & $44.99_{1.12}$ & $38.10_{0.83}$ & $81.16_{1.10}$ & $49.88_{1.02}$ & $51.81_{0.19}$ & $63.70_{0.19}$ & $50.36_{0.90}$ \\
			UEO & $60.19_{1.82}$ & $7.23_{1.80}$ & $43.66_{2.05}$ & $40.60_{1.55}$ & $81.71_{0.88}$ & $49.21_{2.33}$ & $60.78_{1.82}$ & $53.76_{1.38}$ & $49.64_{1.70}$ \\
			M-Tuning & $64.16_{0.79}$ & $17.02_{0.68}$ & $50.46_{0.39}$ & $43.74_{0.55}$ & $86.89_{0.26}$ & $55.22_{0.21}$ & $58.65_{0.95}$ & $65.07_{0.47}$ & $55.15_{0.54}$ \\
			BGAT & $57.72_{1.33}$ & $17.15_{0.06}$ & $47.27_{0.87}$ & $39.15_{1.19}$ & $81.87_{0.41}$ & $55.50_{0.78}$ & $51.87_{0.64}$ & $64.80_{0.25}$ & $51.92_{0.69}$ \\
			\cellcolor{Gray}LSR (Ours) & \cellcolor{Gray}$\textbf{65.50}_{0.42}$ & \cellcolor{Gray}$\textbf{18.92}_{0.18}$ & \cellcolor{Gray}$\textbf{53.00}_{0.50}$ &
			\cellcolor{Gray}$\textbf{45.32}_{0.52}$ & \cellcolor{Gray}$\textbf{87.53}_{0.33}$ & \cellcolor{Gray}$\textbf{59.68}_{0.29}$ &
			\cellcolor{Gray}$\textbf{65.26}_{0.41}$ & \cellcolor{Gray}$\textbf{66.53}_{0.14}$ & \cellcolor{Gray}$\textbf{57.72}_{0.35}$ \\
			\bottomrule[1.1pt]
	\end{tabular}}
\end{table*}
\vspace{-0.4cm}

\subsection{Ablation Studies}
Table~\ref{tab:ablation_study} examines the contributions of the key components in our method at $\gamma=0.4$. We report all-class accuracy of different ablation variants, with each result expressed as $\mathrm{mean}_{\mathrm{std}}$. The variant trained solely with $\mathcal{L}_{\mathrm{P}}$ improves average accuracy slightly, from 55.16\% to 55.49\%. This limited improvement stems from supervision restricted to the adaptation vocabulary while neglecting to preserve omitted-class recognition. Incorporating $\mathcal{L}_{\mathrm{R}}$ extends supervision to the latent label space reconstructed from CMSP-generated representations. Thus, we evaluate the contribution of each loss used to train the CMSP. Training the CMSP with  only $\mathcal{L}_{\mathrm{cls}}$ or $\mathcal{L}_{\mathrm{feat}}$ achieves average accuracies of 56.97\% and 56.69\%, respectively. Both variants outperform prompt tuning with $\mathcal{L}_{\mathrm{P}}$ alone, supporting the benefit of reconstruction-based supervision. Combining both losses improves average accuracy to \textbf{57.54}\% and yields the best results on all six datasets. This suggests that $\mathcal{L}_{\mathrm{cls}}$ and $\mathcal{L}_{\mathrm{feat}}$ complement each other, enabling more effective latent label space reconstruction. Appendix \ref{app:more_ablation_studies} also compares the omitted-class accuracy between variants with and without $\mathcal{L}_{\mathrm{R}}$.

\begin{table*}[ht] 
	\centering 
	\caption{Comparison of all-class accuracy (\%) for ablation variants at $\gamma=0.4$. Values are expressed as $\mathrm{mean}_{\mathrm{std}}$. The highest accuracies are bold.} 
	\vspace{-0.2cm}
	\label{tab:ablation_study} 
	\small
	\renewcommand{\arraystretch}{1.1} 
	\resizebox{1.0\linewidth}{!}{ 
		\begin{tabular}{cccc|ccccccc} 
			\toprule[1.1pt] 
			$\mathcal{L}_\mathrm{P}$ & $\mathcal{L}_\mathrm{R}$ & $\mathcal{L}_\mathrm{feat}$ & $\mathcal{L}_\mathrm{cls}$ & Flw. & FA. & CUB. & Pets. & Cars. & RES. & Avg. \\  
			\midrule 
			\textbf{\ding{55}} & \textbf{\ding{55}} & \textbf{\ding{55}} & \textbf{\ding{55}} & $63.63_{0.00}$ & $17.64_{0.00}$ & $51.77_{0.00}$ & $84.51_{0.00}$ & $58.93_{0.00}$ & $54.48_{0.00}$ & $55.16_{0.00}$  \\  
			\textbf{\checkmark} & \textbf{\ding{55}} & \textbf{\ding{55}} & \textbf{\ding{55}} & $62.45_{0.65}$ & $17.16_{0.56}$ & $50.36_{0.99}$ & $83.19_{0.89}$ & $57.07_{0.47}$ & $62.72_{0.55}$ & $55.49_{0.69}$ \\  
			\textbf{\checkmark} & \textbf{\checkmark} & \textbf{\ding{55}} & \textbf{\checkmark} & $63.79_{1.16}$ & $18.33_{0.26}$ & $51.70_{0.48}$ & $85.75_{0.60}$ & $58.78_{0.13}$ & $63.49_{0.30}$ & $56.97_{0.49}$ \\ 
			\textbf{\checkmark} & \textbf{\checkmark} & \textbf{\checkmark} & \textbf{\ding{55}} & $63.12_{0.49}$ & $18.07_{0.65}$ & $52.02_{0.36}$ & $85.78_{0.78}$ & $58.58_{0.25}$ & $62.56_{0.57}$ & $56.69_{0.52}$ \\ 
			\cellcolor{Gray}\textbf{\checkmark} & \cellcolor{Gray}\textbf{\checkmark} & \cellcolor{Gray}\textbf{\checkmark} & \cellcolor{Gray}\textbf{\checkmark} & \cellcolor{Gray}$\textbf{64.36}_{0.72}$ & \cellcolor{Gray}$\textbf{18.55}_{0.04}$ & \cellcolor{Gray}$\textbf{52.43}_{0.18}$ & \cellcolor{Gray}$\textbf{86.33}_{0.54}$ & \cellcolor{Gray}$\textbf{59.06}_{0.48}$ & \cellcolor{Gray}$\textbf{64.52}_{0.39}$ & \cellcolor{Gray}$\textbf{57.54}_{0.39}$ \\  
			\bottomrule[1.1pt] 
	\end{tabular}} 
\end{table*}
\vspace{-0.2cm}

\subsection{More Analysis}
\label{sec:more_analysis}
\textbf{Our method mitigates the adaptation bias}, as evidenced by consistent improvements in average omitted-group recall (OGR) and omitted-class accuracy (OCA) across eight datasets. Higher OGR indicates fewer omitted-class samples being misclassified as provided classes, and higher OCA demonstrates better fine-grained recognition of omitted classes. Compared with the best adaptation method for each metric, LSR improves OGR and OCA by \textbf{5.06} and \textbf{5.72} percentage points at $\gamma=0.2$, \textbf{4.66} and \textbf{4.38} percentage points at $\gamma=0.3$, and \textbf{4.22} and \textbf{5.08} percentage points at $\gamma=0.4$, respectively. Moreover, LSR largely preserves the zero-shot CLIP baseline's performance in both metrics and surpasses it in several settings. These results support the effectiveness of reconstructed semantic supervision in reducing adaptation bias toward the provided classes and preserving omitted-class recognition.

\vspace{-0.1cm}
\begin{figure*}[ht]
	\centering
	\begin{subfigure}[t]{0.32\textwidth}
		\centering
		\includegraphics[width=\linewidth]{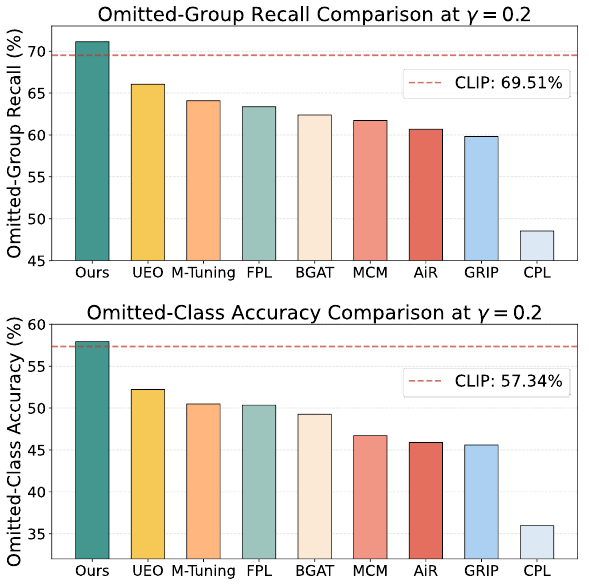}
		\caption{$\gamma=0.2$}
		\label{fig:0.2_avg}
	\end{subfigure}
	\hfill
	\begin{subfigure}[t]{0.32\textwidth}
		\centering
		\includegraphics[width=\linewidth]{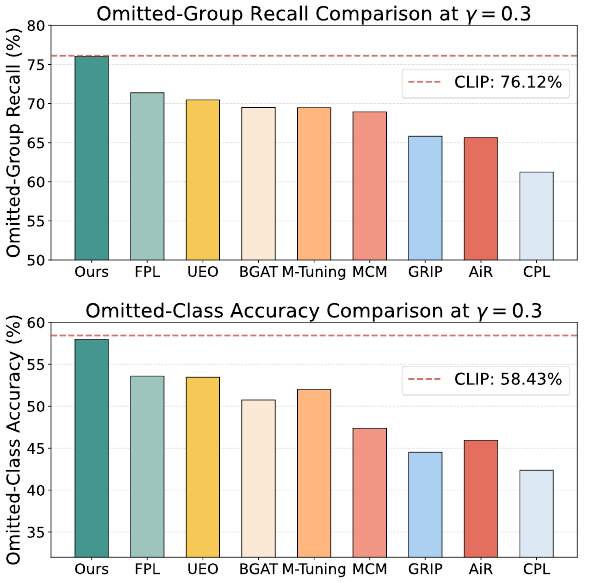}
		\caption{$\gamma=0.3$}
		\label{fig:0.3_avg}
	\end{subfigure}
	\hfill
	\begin{subfigure}[t]{0.32\textwidth}
		\centering
		\includegraphics[width=\linewidth]{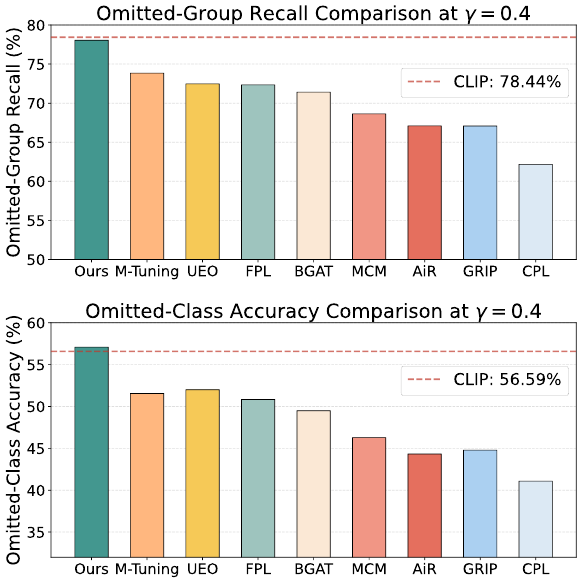}
		\caption{$\gamma=0.4$}
		\label{fig:0.4_avg}
	\end{subfigure}
	\caption{Average results of omitted-group recall (\%) and omitted-class accuracy (\%) across eight datasets at different omission ratios $\gamma$.}
	\vspace{-0.1cm}
	\label{fig:avg_omitted_results}
\end{figure*}

\textbf{The reconstructed latent label space contains semantic anchors derived from omitted-class images which captures the majority of omitted-class samples.}
As shown in Fig.~\ref{fig:semantic_anchors_ratio}, the proportion of anchors derived from omitted-class images increases with the omission ratio. Table~\ref{tab:semantic_anchors} further demonstrates that these anchors cover 65.00\%\textasciitilde{}91.66\% of omitted classes and capture 64.49\%\textasciitilde{}87.95\% of omitted-class samples. Together, these results support the ability of LSR to incorporate omitted-class semantics into the reconstructed latent label space. During prompt tuning, these anchors serve as prediction targets, and $\mathcal{L}_{\mathrm{R}}$ encourages discrimination among the corresponding latent groups. The effectiveness of our design is further supported by its consistent advantage over simple WordNet-based vocabulary expansion, as demonstrated in Appendix~\ref{app:wordnets}.

\begin{figure}[ht]
	\begin{minipage}[t]{0.47\textwidth}
		\centering
		\vspace{-0.2cm}
		\includegraphics[width=1.0\linewidth, height=0.63\linewidth]
		{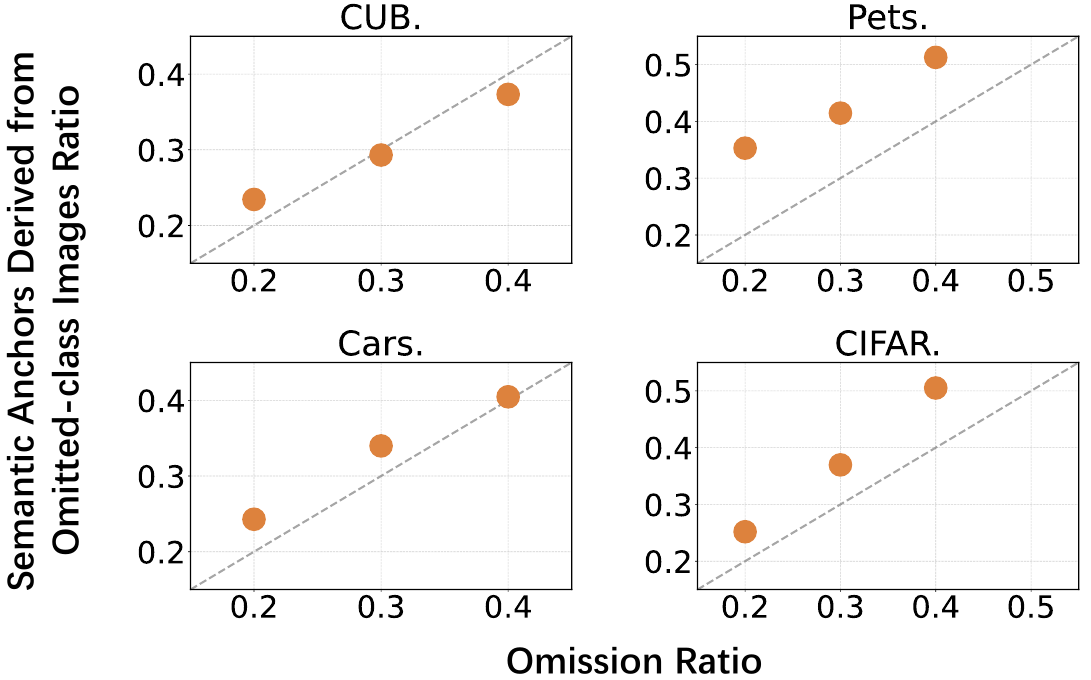}
		\caption{The proportion of semantic anchors derived from omitted-class images at different omission ratios.}
		\vspace{-0.2cm}
		\label{fig:semantic_anchors_ratio}
	\end{minipage}
	\hfill
	\begin{minipage}[t]{0.50\textwidth}
		\centering
		\captionof{table}{The omitted classes coverage ratio (\%) and omitted-class samples capture ratio (\%) of all semantic anchors derived from omitted-class images at different $\gamma$.}
		\label{tab:semantic_anchors}
		\small
		\renewcommand{\arraystretch}{1.1}
		\resizebox{1.0\linewidth}{!}{ 
			\begin{tabular}{cc|cccc}
				\toprule[1.1pt]
				Metrics & $\gamma$ & CUB. & Pets. & Cars. & CIFAR. \\
				\midrule
				\multirow{4}{*}{\shortstack[c]{Omitted classes \\ coverage ratio}}
				& 0.2 & 82.50 & 87.50 & 87.50 & 80.00 \\
				& 0.3 & 75.00 & 91.66 & 83.05 & 83.33 \\
				& 0.4 & 65.00 & 86.66 & 65.82 & 82.50 \\
				& Avg. & 74.17 & 88.61 & 78.79 & 81.94 \\
				\midrule
				\midrule
				\multirow{4}{*}{\shortstack[c]{Omitted-class \\ samples capture \\ ratio}}
				& 0.2 & 86.89 & 79.84 & 66.16 & 64.49 \\
				& 0.3 & 81.74 & 86.90 & 72.49 & 71.51 \\
				& 0.4 & 86.89 & 87.95 & 75.15 & 82.11 \\
				& Avg. & 85.17 & 84.90 & 71.27 & 72.70 \\
				\bottomrule[1.1pt]
			\end{tabular}}
	\end{minipage}
\end{figure}

\textbf{Evaluation with a larger image encoder.}
We evaluate LSR with a larger image encoder ViT-L/14 on five datasets. As shown in Table~\ref{tab:ViT_L_14}, LSR achieves the highest accuracy on all five datasets, exceeding the strongest baseline by an average of \textbf{1.89} percentage points. These results demonstrate that our method remains effective with a larger visual backbone.

\textbf{Evaluation with a stronger foundation model.}
We further evaluate LSR with PE-Core-B/16~\citep{PE-core} to assess its effectiveness on a stronger foundation model. Table~\ref{tab:PE_core} shows that LSR consistently achieves the highest accuracy across all five datasets, reaching an average of 73.96\%. Specifically, our method outperforms the strongest baseline on RES and DTD by \textbf{3.32} and \textbf{2.07} percentage points, respectively. These results illustrate the robustness of our method across different foundation models.

\begin{table}[h]
	\begin{minipage}{0.495\textwidth}
		\centering
		\caption{Comparison results of all-class accuracy (\%) with the image encoder ViT-L/14.}
		\label{tab:ViT_L_14}
		\renewcommand{\arraystretch}{1.2}
		\begin{adjustbox}{width=1\columnwidth, center}
			\begin{tabular}{l|ccccc}
				\toprule[1.1pt]
				Method & Flw. & FA. & DTD. & RES. & CIFAR.   \\
				\cmidrule(lr){1-6}
				CLIP & $73.94_{0.00}$ & $30.51_{0.00}$ & $52.76_{0.00}$ & $62.63_{0.00}$ & $75.93_{0.00}$  \\
				FPL & $76.35_{1.23}$ & $30.65_{0.42}$ & $52.62_{1.33}$ & $68.75_{0.19}$ & $79.34_{0.19}$  \\
				GRIP & $70.23_{0.60}$ & $26.28_{1.53}$ & $48.49_{0.67}$ & $71.90_{0.86}$ & $76.76_{0.49}$  \\
				CPL & $65.85_{5.82}$ & $28.52_{1.13}$ & $49.79_{0.88}$ & $70.07_{0.68}$ & $75.79_{0.15}$  \\
				AiR & $66.09_{1.05}$ & $26.07_{1.20}$ & $48.90_{3.09}$ & $69.05_{0.74}$ & $76.69_{0.17}$  \\
				MCM & $71.32_{2.65}$ & $27.18_{1.11}$ & $52.25_{0.94}$ & $62.84_{1.21}$ & $78.53_{0.32}$  \\
				UEO & $75.42_{0.34}$ & $27.86_{0.69}$ & $53.07_{0.32}$ & $68.24_{1.73}$ & $78.71_{0.36}$  \\
				M-T & $75.11_{1.01}$ & $30.16_{1.04}$ & $52.15_{0.35}$ & $68.64_{0.87}$ & $78.87_{0.23}$  \\
				BGAT & $71.92_{1.79}$ & $29.21_{0.83}$ & $52.59_{0.29}$ & $60.72_{0.78}$ & $78.49_{0.21}$  \\
				\cellcolor{Gray}Ours & 
				\cellcolor{Gray}$\textbf{76.68}_{0.64}$ & \cellcolor{Gray}$\textbf{33.53}_{0.25}$ & \cellcolor{Gray}$\textbf{54.63}_{0.14}$ & \cellcolor{Gray}$\textbf{72.49}_{0.69}$ & \cellcolor{Gray}$\textbf{79.80}_{0.07}$ \\
				\bottomrule[1.1pt]
			\end{tabular}
		\end{adjustbox}
	\end{minipage}
	\hspace{0.1cm}
	\begin{minipage}{0.495\textwidth}
		\centering
		\caption{Comparison results of all-class accuracy (\%) on the foundation model PE-Core-B/16.}
		\label{tab:PE_core}
		\renewcommand{\arraystretch}{1.2}
		\begin{adjustbox}{width=1\columnwidth, center}
			\begin{tabular}{l|ccccc}
				\toprule[1.1pt]
				Method & Flw. & FA. & DTD. & RES. & CIFAR.   \\
				\cmidrule(lr){1-6}
				CLIP & $84.40_{0.00}$ & $55.20_{0.00}$ & $65.69_{0.00}$ & $67.78_{0.00}$ & $82.47_{0.00}$  \\
				FPL & $85.82_{0.93}$ & $54.72_{1.15}$ & $66.60_{0.44}$ & $70.22_{0.29}$ & $82.86_{0.20}$  \\
				GRIP & $80.48_{0.71}$ & $44.96_{1.24}$ & $65.39_{1.46}$ & $68.98_{1.08}$ & $81.27_{0.60}$  \\
				CPL & $83.14_{0.72}$ & $44.04_{0.84}$ & $64.79_{1.35}$ & $68.72_{1.48}$ & $81.75_{0.28}$  \\
				AiR & $76.84_{2.36}$ & $40.79_{0.57}$ & $64.88_{1.12}$ & $67.00_{1.32}$ & $80.87_{0.44}$  \\
				MCM & $82.21_{1.10}$ & $40.79_{0.95}$ & $64.70_{1.24}$ & $64.05_{0.64}$ & $83.02_{0.02}$  \\
				UEO & $82.54_{1.09}$ & $46.19_{1.25}$ & $66.74_{1.01}$ & $71.14_{1.50}$ & $83.62_{0.19}$  \\
				M-T & $85.40_{0.25}$ & $53.84_{0.54}$ & $65.80_{0.55}$ & $69.25_{0.92}$ & $82.94_{0.18}$  \\
				BGAT & $79.67_{0.57}$ & $50.63_{0.23}$ & $63.88_{0.56}$ & $65.67_{0.99}$ & $82.13_{0.21}$  \\
				\cellcolor{Gray}Ours & 
				\cellcolor{Gray}$\textbf{86.37}_{0.39}$ & \cellcolor{Gray}$\textbf{56.48}_{1.14}$ & \cellcolor{Gray}$\textbf{68.81}_{0.43}$ & \cellcolor{Gray}$\textbf{74.46}_{0.92}$ & \cellcolor{Gray}$\textbf{83.68}_{0.29}$ \\
				\bottomrule[1.1pt]
			\end{tabular}
		\end{adjustbox}
	\end{minipage}
\end{table}

\textbf{Sensitivity analysis.} We evaluate the sensitivity of LSR to $\lambda$ (in Eq.~\ref{eq:loss_all}) and $\varepsilon$ (in Appendix \ref{app:pseudo_labeled_dataset_generation}) on CIFAR and Cars. As shown in Fig.~\ref{fig:sensitivity_analysis_lambda} and Fig.~\ref{fig:sensitivity_analysis_epsilon}, all-class accuracy remains   relatively stable as $\lambda$ varies from 0.2 to 5.0 and $\varepsilon$ varies from 0.2 to 0.4. These results indicate that LSR is robust to the choice of these hyperparameters. Sensitivity analyses on more datasets and parameters are provided in Appendix \ref{app:sensitivity_analysis}.

\textbf{Comparison under the conventional UVA setting.} As shown in Fig.~\ref{fig:UVA_Comparison}, LSR achieves the second-highest all-class accuracy under the conventional UVA setting. The adaptation vocabulary is complete, reducing the benefit of reconstructing a latent label space. Nevertheless, the reconstructed latent label space provides supplementary supervision derived from unlabeled images, which can complement vocabulary-based pseudo-labels and help LSR maintain competitive performance.

\begin{figure*}[ht]
	\centering
	\begin{subfigure}[t]{0.26\textwidth}
		\centering
		\includegraphics[width=\linewidth]{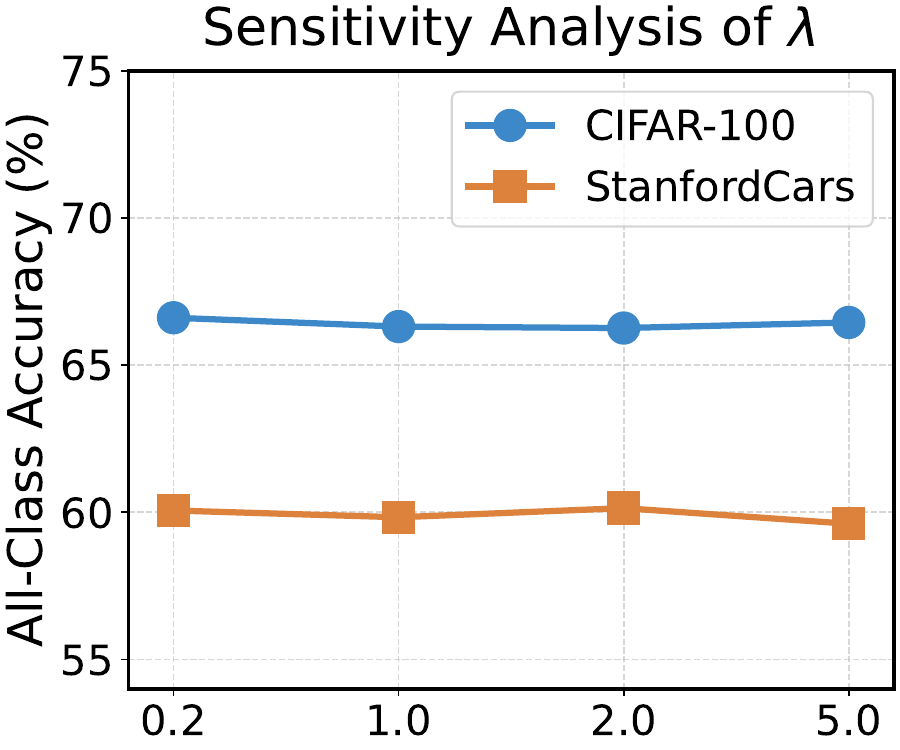}
		\vspace{-0.2cm}
		\caption{}
		\label{fig:sensitivity_analysis_lambda}
	\end{subfigure}
	\hfill
	\begin{subfigure}[t]{0.26\textwidth}
		\centering
		\includegraphics[width=\linewidth]{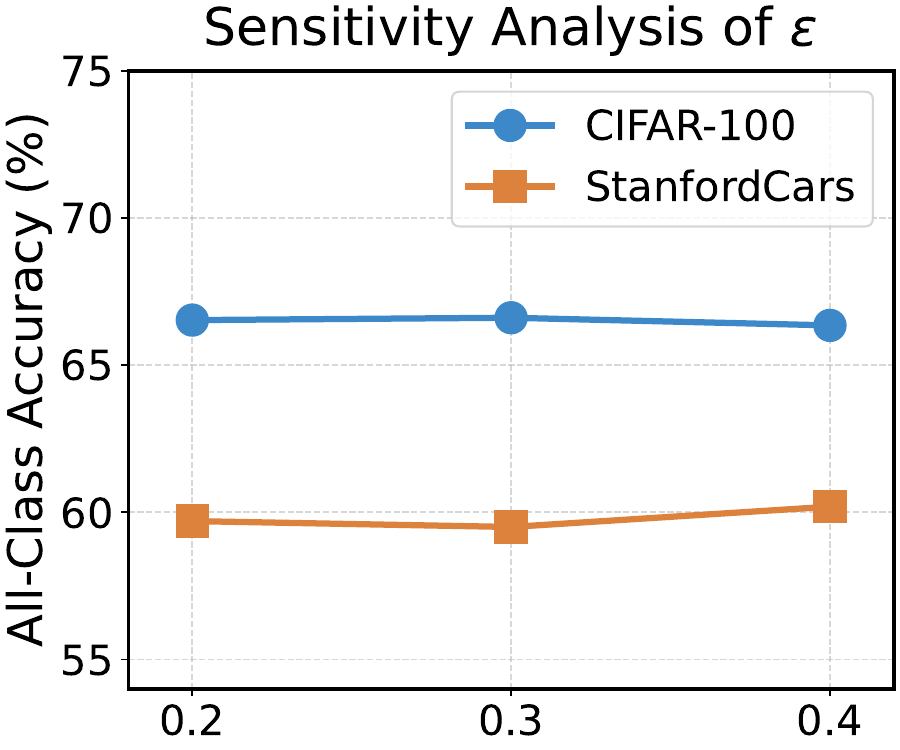}
		\vspace{-0.2cm}
		\caption{}
		\label{fig:sensitivity_analysis_epsilon}
	\end{subfigure}
	\hfill
	\begin{subfigure}[t]{0.44\textwidth}
		\centering
		\includegraphics[width=\linewidth]{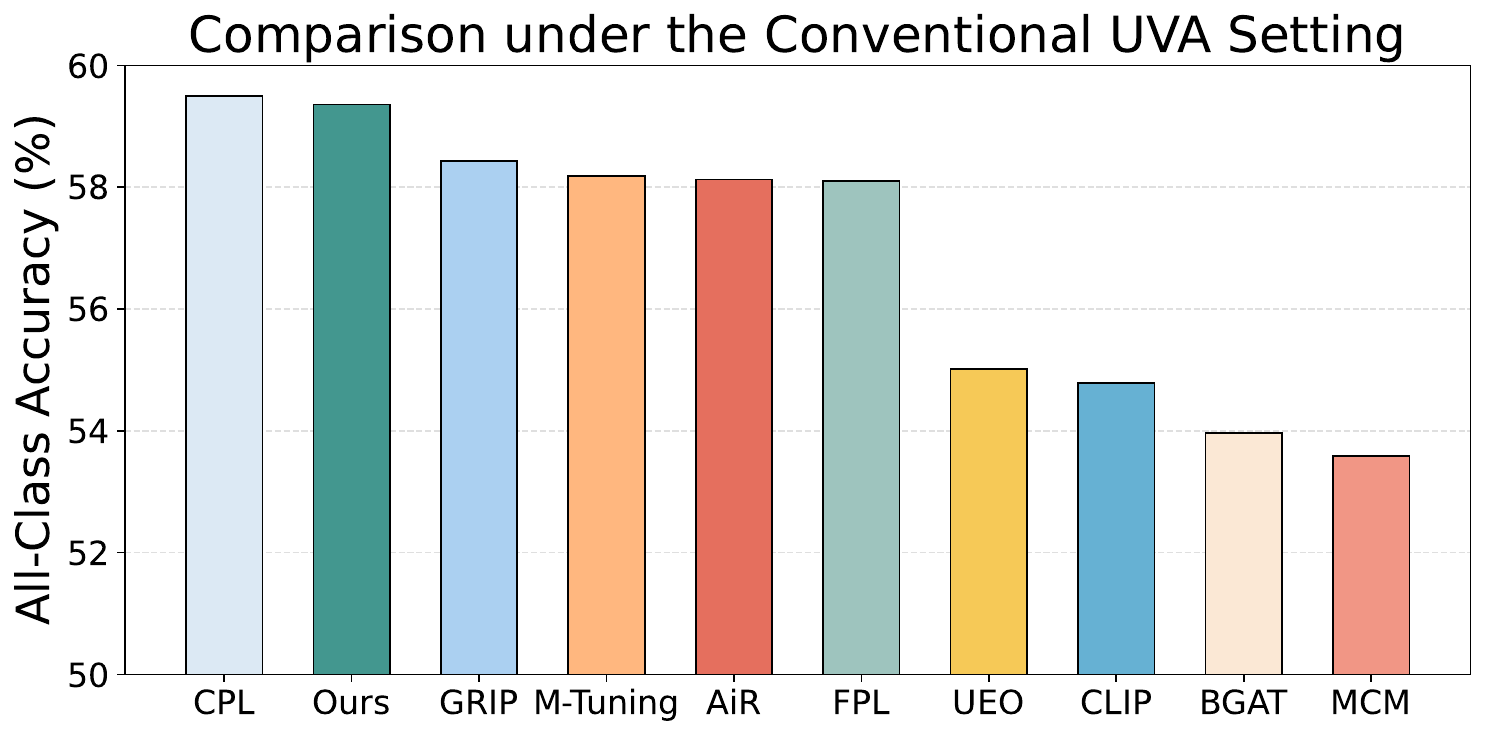}
		\vspace{-0.4cm}
		\caption{}
		\label{fig:UVA_Comparison}
	\end{subfigure}
	\vspace{-0.2cm}
	\caption{(a) Sensitivity analysis of $\lambda$. (b) Sensitivity analysis of $\varepsilon$. (c) Comparison under the conventional UVA setting.}
\end{figure*}

\section{Conclusion}
In this work, we propose vocabulary-incomplete unsupervised vision-language model adaptation (VI-UVA), a practical yet challenging setting where only a subset of class names is available during adaptation. We identify the adaptation bias induced by vocabulary incompleteness, which leads to vocabulary overfitting and degrades omitted-class recognition. To address this issue, we propose label space reconstruction (LSR), which reconstructs the latent semantic structures from unlabeled data and incorporates them into prompt tuning. By extending supervision beyond the adaptation vocabulary, LSR effectively mitigates vocabulary overfitting and preserves omitted-class recognition. We also derive an omitted-class semantic discrepancy bound for LSR and identify the key factors affecting the bound. Extensive experiments on eight benchmarks show that LSR consistently outperforms existing methods at different omission ratios and remains robust across different visual backbones and foundation models. 

\subsection*{AI use statement}

In this work, we used generative AI tools solely to assist with grammar, spelling, and sentence clarity. We have reviewed all AI-assisted work and bear full responsibility for all content generated by the LLM. The generative AI contributed in no other way to the paper. The research idea, experiment design, and all other content were developed and completed by  ourselves.

\subsection*{Ethics statement}

We confirm that our work on vocabulary-incomplete unsupervised vision-language model adaptation adheres to the ICLR Code of Ethics. The research utilizes publicly available datasets containing no personally identifiable information. As a strictly technical contribution focused on new learning paradigm and algorithmic design, we foresee no ethical issues or potential for misuse.

\subsection*{Reproducibility statement}

The complete code for implementing data preprocessing, model training, and evaluation is provided in the Supplementary Material to facilitate the reproduction of our results. Also, our experimental setup is described in Sec.~\ref{sec:experimental_settings} and Appendix \ref{app:implementation_details}.

\bibliography{iclr2027_conference}
\bibliographystyle{iclr2027_conference}

\newpage
\appendix

\section*{Table of Contents for the Appendix}
\begin{itemize}[leftmargin=25pt]
	\setlength{\itemsep}{10pt}
	\item ~\ref{app:comparison_with_related_learning_paradigms}: Related Learning Paradigms  \dotfill \pageref{app:comparison_with_related_learning_paradigms}
	
	\item ~\ref{app:overview_and_pseudo-code}: Overview and Pseudo-Code of LSR \dotfill \pageref{app:overview_and_pseudo-code}
	
	\item ~\ref{app:proof}: Proof and further analysis of Theorem \ref{thm1:semantic_discrepancy}   \dotfill \pageref{app:proof}
	
	\item ~\ref{app:datasets_and_evaluation}: Datasets and Evaluations \dotfill \pageref{app:datasets_and_evaluation}
	
	\item ~\ref{app:comparison_methods}: Comparison Methods \dotfill \pageref{app:comparison_methods}
	
	\item ~\ref{app:implementation}: Implementation \dotfill \pageref{app:implementation}
	
		\item ~~~~~~\ref{app:pseudo_labeled_dataset_generation}: Pseudo-Labeled Dataset Generation \dotfill \pageref{app:pseudo_labeled_dataset_generation}
	
		\item ~~~~~~\ref{app:filtering_procedure_of_unreliable_centers}: Filtering Procedure of Unreliable Anchors  \dotfill \pageref{app:filtering_procedure_of_unreliable_centers}
	
		\item ~~~~~~\ref{app:implementation_details}: Implementation Details \dotfill \pageref{app:implementation_details}
	
	\item ~\ref{app:more_experiment_results_and_analysis}: More Experiment Results and Analysis \dotfill \pageref{app:more_experiment_results_and_analysis}
		\item ~~~~~~\ref{app:additional_main_results}: More Main Results \dotfill \pageref{app:additional_main_results}
		
		\item ~~~~~~\ref{app:more_ablation_studies}:  More Ablation Studies \dotfill \pageref{app:more_ablation_studies}
		
		\item ~~~~~~\ref{app:wordnets}:  Comparison with WordNet-Based Vocabulary Expansion \dotfill \pageref{app:wordnets}
	
		\item ~~~~~~\ref{app:sensitivity_analysis}: Sensitivity Analysis \dotfill \pageref{app:sensitivity_analysis}
		
	\item ~\ref{app:compute_resources}: Compute Resources \dotfill \pageref{app:compute_resources}
	
\end{itemize}

\newpage

\section{Related Learning Paradigms}
\label{app:comparison_with_related_learning_paradigms}
In Table \ref{tab:paradigm_comparison}, we compare vocabulary-incomplete unsupervised VLM adaptation (VI-UVA) with open-set recognition (OSR), out-of-distribution detection (OOD), and generalized category discovery (GCD). VI-UVA differs from these related paradigms in terms of both adaptation data and evaluation objectives. During adaptation, VI-UVA requires no labeled samples, whereas OSR and GCD rely on labeled samples from known classes. During evaluation, VI-UVA performs classification over all classes. In contrast, OSR rejects all unknown-class samples as a single unknown category, OOD detection performs binary classification between ID and OOD samples, and GCD assigns known-class samples to known categories while grouping novel-class samples into unnamed clusters. 
\vspace{-0.2cm}
\begin{table*}[ht] 
	\centering 
	\caption{Comparison of VI-UVA with related learning paradigms.} 
	\vspace{-0.2cm}
	\label{tab:paradigm_comparison} 
	\small
	\setlength{\tabcolsep}{4pt} 
	\renewcommand{\arraystretch}{1.5} 
	\begin{tabularx}{\linewidth}{ 
			>{\raggedright\arraybackslash}p{0.08\linewidth} 
			>{\raggedright\arraybackslash}p{0.45\linewidth} 
			>{\raggedright\arraybackslash}X } 
		\toprule 
		Paradigm 
		& Adaptation data 
		& Evaluation objectives \\ 
		\midrule 
		
		OSR
		& Labeled data from known classes. 
		& Classify known samples and reject all unknown samples as a single unknown category. \\ 
		
		OOD
		& Labeled or unlabeled ID data.
		& Distinguish OOD samples from ID samples without recognizing their semantic categories. \\ 
		
		GCD
		& Labeled data from known classes and unlabeled data from both known and novel classes. 
		& Classify known classes and cluster novel classes. \\ 
		
		\textbf{VI-UVA} 
		& Unlabeled data from all classes and an incomplete adaptation vocabulary of class names. 
		& Classify all classes with the complete vocabulary. \\ 
		
		\bottomrule 
	\end{tabularx} 
\end{table*}
\vspace{-0.4cm}

\section{Overview and Pseudo-Code of LSR}
Fig.~\ref{fig:overview} presents the overview of LSR. The pseudo-code of LSR is shown in Algorithm \ref{alg:LSR}.
\label{app:overview_and_pseudo-code}
\begin{figure*}[ht]
	\centering
	\includegraphics[width=1.0\textwidth]{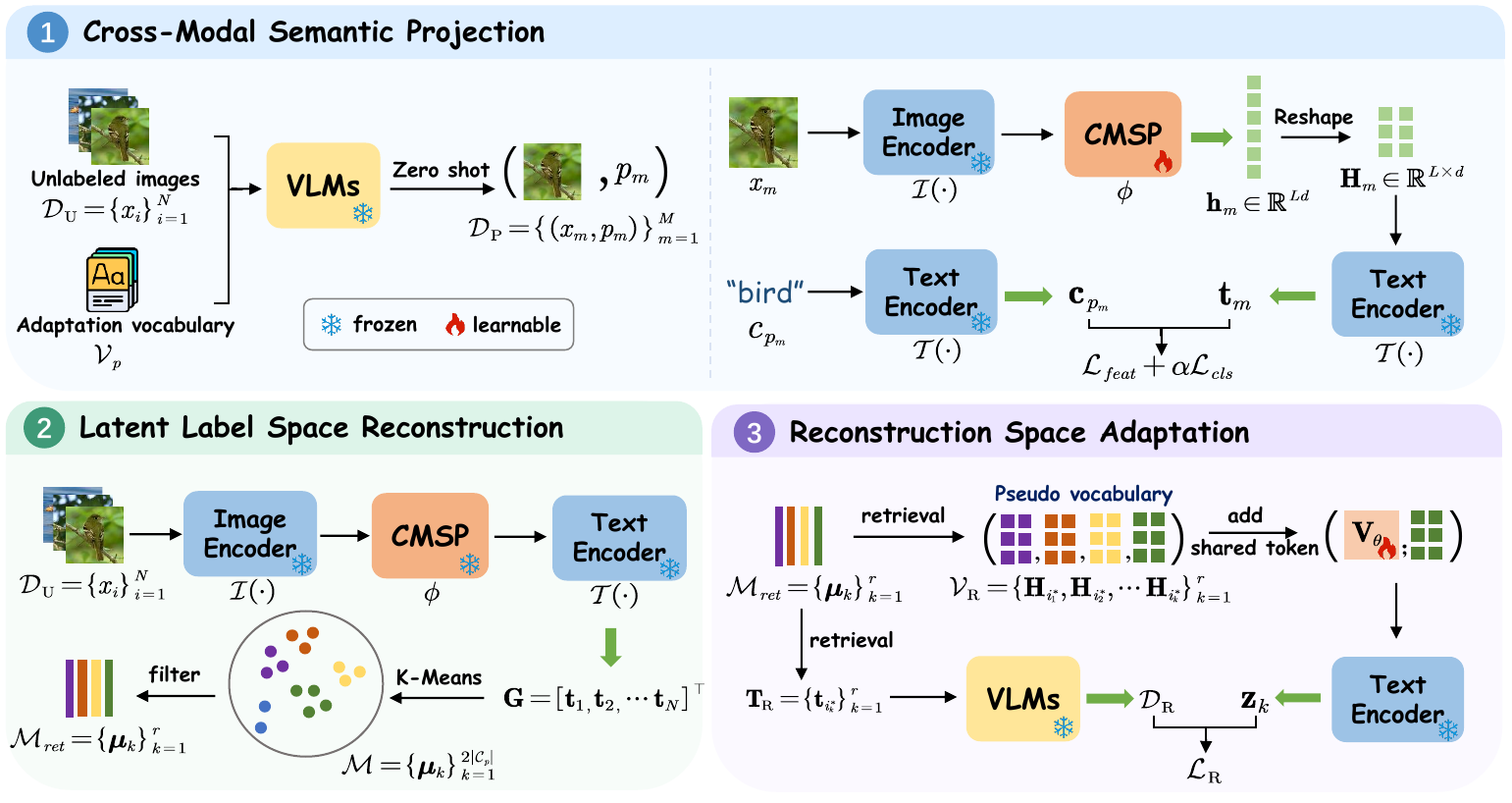}
	\caption{\textit{First}, we train a cross-modal semantic projector (CMSP) using a pseudo-labeled dataset derived from the adaptation vocabulary, mapping visual representations into the textual embedding space. \textit{Second}, we apply the CMSP to all unlabeled images to generate textual embeddings, cluster them into semantic anchors, and filter out unreliable anchors. \textit{Third}, for each retained anchor, we retrieve the corresponding pseudo-token sequence and textual representation to construct a pseudo-vocabulary and a pseudo-labeled dataset, respectively. Then, we perform prompt tuning with the pseudo-vocabulary and pseudo-labeled dataset to preserve omitted-class recognition.}
	\label{fig:overview}
\end{figure*}
\vspace{-0.2cm}

\begin{algorithm}[!ht]
	\caption{Training process of LSR}
	\label{alg:LSR}
	\begin{algorithmic}[1]
		\STATE \textbf{Input}: Unlabeled dataset $\mathcal{D}_{\mathrm{U}}$, adaptation vocabulary $\mathcal{V}_p$, pre-trained image encoder $\mathcal{I}$ and text encoder $\mathcal{T}$, training epochs $E_{\mathrm{C}}$ and $E_{\mathrm{P}}$.
		\STATE \textbf{Output}: Shared learnable context tokens $\mathbf{V}_{\theta}$.
		\STATE Freeze the parameters of $\mathcal{I}$ and $\mathcal{T}$.
		
		\STATE \textbf{Stage 1: Train the cross-modal semantic projector.}
		\STATE Extract visual embeddings of $\mathcal{D}_{\mathrm{U}}$ and textual embeddings of $\mathcal{V}_p$.
		\STATE Construct a balanced pseudo-labeled dataset $\mathcal{D}_{\mathrm{P}}$ with $\mathcal{V}_p$.
		\STATE Initialize the CMSP parameters $\phi$.
		\FOR{epoch $=1,\dots,E_{\mathrm{C}}$}
		\FOR{each mini-batch from $\mathcal{D}_{\mathrm{P}}$}
		\STATE Compute $\mathcal{L}_{\mathrm{CMSP}}=\mathcal{L}_{\mathrm{feat}}+\alpha\mathcal{L}_{\mathrm{cls}}$.
		\STATE Update only $\phi$ via gradient descent.
		\ENDFOR
		\ENDFOR
		
		\STATE \textbf{Stage 2: Reconstruct the latent label space.}
		\STATE Freeze the CMSP parameters $\phi$.
		\STATE Generate the pseudo-token sequence $\mathbf{H}_i$ and textual embedding $\mathbf{t}_i$ for each $x_i\in\mathcal{D}_{\mathrm{U}}$.
		\STATE Perform K-means clustering on $\{\mathbf{t}_i\}_{i=1}^{N}$ into $2|\mathcal{V}_p|$ anchors.
		\STATE Filter unreliable anchors based on their image support to obtain $\mathcal{M}_{\mathrm{ret}}$.
		
		\STATE \textbf{Stage 3: Optimize the shared context tokens.}
		\FOR{each retained anchor $\boldsymbol{\mu}_k\in\mathcal{M}_{\mathrm{ret}}$}
		\STATE Retrieve $i_k^*=\operatorname*{arg\,max}_{i\in\{1,\dots,N\}}\mathrm{sim}(\boldsymbol{\mu}_k,\mathbf{t}_i)$.
		\ENDFOR
		\STATE Construct $\mathcal{V}_{\mathrm{R}}=\{\mathbf{H}_{i_k^*}\}_{k=1}^{r}$ and $\mathbf{T}_{\mathrm{R}}=\{\mathbf{t}_{i_k^*}\}_{k=1}^{r}$.
		\STATE Construct a balanced pseudo-labeled dataset $\mathcal{D}_{\mathrm{R}}$ with $\mathbf{T}_{\mathrm{R}}$.
		\STATE Initialize the shared context tokens $\mathbf{V}_{\theta}$ and keep pseudo-vocabulary $\mathcal{V}_{\mathrm{R}}$ fixed.
		\FOR{epoch $=1,\dots,E_{\mathrm{P}}$}
		\FOR{each pair of mini-batches from $\mathcal{D}_{\mathrm{P}}$ and $\mathcal{D}_{\mathrm{R}}$}
		\STATE Construct prompt-conditioned textual prototypes using $\mathbf{V}_{\theta}$ for $\mathcal{V}_p$ and $\mathcal{V}_{\mathrm{R}}$.
		\STATE Compute $\mathcal{L}_{\mathrm{all}}=\lambda\mathcal{L}_{\mathrm{R}}+\mathcal{L}_{\mathrm{P}}$.
		\STATE Update only $\mathbf{V}_{\theta}$ via gradient descent.
		\ENDFOR
		\ENDFOR
	\end{algorithmic}
\end{algorithm}

\section{Proof and further analysis of Theorem \ref{thm1:semantic_discrepancy}}
\label{app:proof}
In this section, we prove Theorem~\ref{thm1:semantic_discrepancy}. Let $p=|\mathcal{C}_p|$ and $C=|\mathcal{C}_t|$, with
provided classes indexed by $\left\{1,\ldots,p\right\}$. Throughout the analysis, the textual prototypes are $\ell_2$-normalized, i.e., $\|\mathbf{c}_j\|_2=1$ for all $j\in\{1,\ldots,C\}$. As in Eq.~\ref{eq:projection},
\begin{equation}
	\mathbf{x}_i=\mathcal{I}(x_i),~~~~
	\mathbf{t}_i=\mathcal{F}(\mathbf{x}_i),~~~~
	\mathcal{F}=\mathcal{T}\circ\mathcal{R}_{L\times d}\circ f_\phi.
\end{equation}
Let $\mathcal{I}_o$ index the omitted-class samples in $\mathcal{D}_U$
and $N_o=|\mathcal{I}_o|$. For $i\in\mathcal{I}_o$, denote
$\mathbf{c}_{y_i}$ as the textual prototype of its true class.
Unsubscripted norms are Euclidean.

\begin{assumption}
	\label{ass:regularity}
	Suppose that $\mathcal{F}$ is $\ell$-Lipschitz and its Jacobian is $\kappa$-Lipschitz on a convex domain $\Omega$ containing the relevant visual
	representations. For finite $\ell,\kappa\geq0$ and all $\mathbf{u},\mathbf{v}\in\Omega$,
	\begin{equation}
		\begin{aligned}
			\|\mathcal{F}(\mathbf{u})-\mathcal{F}(\mathbf{v})\|
			&\leq\ell\|\mathbf{u}-\mathbf{v}\|,\\
			\|D\mathcal{F}(\mathbf{u})-D\mathcal{F}(\mathbf{v})\|_{\mathrm{op}}
			&\leq\kappa\|\mathbf{u}-\mathbf{v}\|,
		\end{aligned}
		\label{eq:regularity}
	\end{equation}
	where $\|\cdot\|_{\mathrm{op}}$ is the induced operator norm.
\end{assumption}

Define the fixed semantic response over $\mathcal{C}_t$ by
\begin{equation}
	q_j(\mathbf{a})=
	\frac{\exp(\mathbf{a}^{\top}\mathbf{c}_j)}
	{\sum_{k=1}^{C}\exp(\mathbf{a}^{\top}\mathbf{c}_k)},
	~~~~q(\mathbf{a})=(q_j(\mathbf{a}))_{j=1}^{C}.
	\label{eq:response}
\end{equation}
Keep $\mathbf{T}_R=\{\mathbf{t}_{i_k^*}\}_{k=1}^{r}$ as
constructed by the cosine retrieval in Sec.~\ref{sec:RSA}. For each
$i\in\mathcal{I}_o$, define
\begin{equation}
	\begin{aligned}
		\widehat{\mathbf{t}}_i
		\in\operatorname*{arg\,min}_{\mathbf{a}\in\mathbf{T}_R}
		\|\mathbf{t}_i-\mathbf{a}\|,~~~~
		d_i=
		\|\mathbf{t}_i-\widehat{\mathbf{t}}_i\|,~~~~
		\mathfrak{r}_o
		=\left(\frac{1}{N_o}\sum_{i\in\mathcal{I}_o}d_i^2\right)^{1/2}.
	\end{aligned}
	\label{eq:nearest_error}
\end{equation}
The discrepancy in Theorem~\ref{thm1:semantic_discrepancy} is defined as a KL divergence:
\begin{equation}
	\mathcal{E}^{\mathrm{rec}}_o
	=\frac{1}{N_o}\sum_{i\in\mathcal{I}_o}
	D_{\mathrm{KL}}\!\left(
	q(\mathbf{c}_{y_i})\,\middle\|\,q(\widehat{\mathbf{t}}_i)
	\right).
	\label{eq:discrepancy}
\end{equation}

\begin{lemma}
	\label{lem:smooth}
	Under Assumption~\ref{ass:regularity}, for any
	$\mathbf{u},\mathbf{v}\in\Omega$,
	\begin{equation}
		\|\mathcal{F}(\mathbf{u})-\mathcal{F}(\mathbf{v})
		-D\mathcal{F}(\mathbf{v})(\mathbf{u}-\mathbf{v})\|
		\leq\frac{\kappa}{2}\|\mathbf{u}-\mathbf{v}\|^2.
		\label{eq:taylor}
	\end{equation}
	Consequently, if $a_s\geq0$, $\sum_s a_s=1$, and
	$\overline{\mathbf{u}}=\sum_s a_s\mathbf{u}_s$ with
	$\mathbf{u}_s\in\Omega$, then
	\begin{equation}
		\left\|\mathcal{F}(\overline{\mathbf{u}})
		-\sum_s a_s\mathcal{F}(\mathbf{u}_s)\right\|
		\leq\frac{\kappa}{2}\sum_s a_s
		\|\mathbf{u}_s-\overline{\mathbf{u}}\|^2.
		\label{eq:convex_recomposition}
	\end{equation}
\end{lemma}
\begin{proof}
	Set $\mathbf{h}=\mathbf{u}-\mathbf{v}$. Convexity keeps the segment
	$\mathbf{v}+t\mathbf{h}$ inside $\Omega$. The fundamental theorem of
	calculus gives
	\begin{equation}
		\mathcal{F}(\mathbf{u})-\mathcal{F}(\mathbf{v})
		-D\mathcal{F}(\mathbf{v})\mathbf{h}=\int_0^1
		\bigl[D\mathcal{F}(\mathbf{v}+t\mathbf{h})
		-D\mathcal{F}(\mathbf{v})\bigr]\mathbf{h}\,dt.
	\end{equation}
	Its norm is at most
	$\int_0^1\kappa t\|\mathbf{h}\|^2dt
	=\kappa\|\mathbf{h}\|^2/2$, proving
	Eq.~\ref{eq:taylor}. Apply this inequality to each pair
	$(\mathbf{u}_s,\overline{\mathbf{u}})$ and take the weighted sum.
	The linear terms cancel because
	$\sum_s a_s(\mathbf{u}_s-\overline{\mathbf{u}})=0$.
	The triangle inequality gives Eq.~\ref{eq:convex_recomposition}.
\end{proof}

For $\mathcal{D}_P=\{(x_m,p_m)\}_{m=1}^{M}$, define
$\mathcal{P}_j=\{m:p_m=j\}$, $n_j=|\mathcal{P}_j|$, and
\begin{equation}
	\begin{aligned}
		\overline{\mathbf{x}}_j
		&=\frac{1}{n_j}\sum_{m\in\mathcal{P}_j}\mathbf{x}_m,\\
		e_j&=\frac{1}{n_j}\sum_{m\in\mathcal{P}_j}
		\|\mathbf{t}_m-\mathbf{c}_j\|_1,\\
		v_j&=\frac{1}{n_j}\sum_{m\in\mathcal{P}_j}
		\|\mathbf{x}_m-\overline{\mathbf{x}}_j\|^2,\\
		\varepsilon_{\mathrm{proj}}&=\max_{1\leq j\leq p}\left(e_j+\frac{\kappa}{2}v_j\right).
	\end{aligned}
	\label{eq:projection_error}
\end{equation}

\begin{lemma}
	\label{lem:projection_error}
	Under Assumption~\ref{ass:regularity}, for every $j\in\left\{1,\ldots,p\right\}$,
	\begin{equation}
		\|\mathcal{F}(\overline{\mathbf{x}}_j)-\mathbf{c}_j\|
		\leq e_j+\frac{\kappa}{2}v_j
		\leq\varepsilon_{\mathrm{proj}}.
	\end{equation}
\end{lemma}
\begin{proof}
	Apply Lemma~\ref{lem:smooth} to the representations in
	$\mathcal{P}_j$ with weights $1/n_j$. Inserting their mean generated
	representation gives
	\begin{equation}
		\begin{aligned}
			\|\mathcal{F}(\overline{\mathbf{x}}_j)-\mathbf{c}_j\|
			&\leq\left\|\mathcal{F}(\overline{\mathbf{x}}_j)
			-\frac{1}{n_j}\sum_{m\in\mathcal{P}_j}\mathbf{t}_m\right\|
			+\left\|\frac{1}{n_j}\sum_{m\in\mathcal{P}_j}
			(\mathbf{t}_m-\mathbf{c}_j)\right\|\\
			&\leq\frac{\kappa}{2}v_j
			+\frac{1}{n_j}\sum_{m\in\mathcal{P}_j}
			\|\mathbf{t}_m-\mathbf{c}_j\|_1\\
			&=\frac{\kappa}{2}v_j+e_j,
		\end{aligned}
	\end{equation}
	where $\|\cdot\|_2\leq\|\cdot\|_1$ gives the second line.
	Taking the maximum over groups proves the lemma.
\end{proof}

Let $\Delta_p=\{\mathbf{w}\in\mathbb{R}^{p}:w_j\geq0,
\sum_{j=1}^{p}w_j=1\}$. For the same $\mathbf{w}\in\Delta_p$, set
\[
\mathbf{x}_{\mathbf{w}}=\sum_{j=1}^{p}w_j\overline{\mathbf{x}}_j,
~~~~~
\mathbf{c}_{\mathbf{w}}=\sum_{j=1}^{p}w_j\mathbf{c}_j.
\]
Define the shared-recomposition discrepancy by
\begin{equation}
	\begin{aligned}
		b_i&=\min_{\mathbf{w}\in\Delta_p}\left\{
		\ell\|\mathbf{x}_i-\mathbf{x}_{\mathbf{w}}\|
		+\|\mathbf{c}_{y_i}-\mathbf{c}_{\mathbf{w}}\|
		+\frac{\kappa}{2}\sum_{j=1}^{p}w_j
		\|\overline{\mathbf{x}}_j-\mathbf{x}_{\mathbf{w}}\|^2\right\},\\
		\mathfrak{b}_o&=\left(\frac{1}{N_o}
		\sum_{i\in\mathcal{I}_o}b_i^2\right)^{1/2}.
	\end{aligned}
	\label{eq:shared_error}
\end{equation}
The three terms of $b_i$ are the visual residual, textual residual, and nonlinear
recomposition remainder. They are evaluated with common weights $\mathbf{w}$, which
may vary across samples. The minimum exists by continuity on the compact
simplex.

\begin{lemma}
	\label{lem:transfer_to_omitted}
	Under Assumption~\ref{ass:regularity}, every $i\in\mathcal{I}_o$ satisfies
	\begin{equation}
		\|\mathbf{t}_i-\mathbf{c}_{y_i}\|
		\leq\varepsilon_{\mathrm{proj}}+b_i.
	\end{equation}
\end{lemma}
\begin{proof}
	For any $\mathbf{w}\in\Delta_p$, apply the triangle inequality and insert
	$\mathcal{F}(\mathbf{x}_{\mathbf{w}})$,
	$\sum_jw_j\mathcal{F}(\overline{\mathbf{x}}_j)$, and
	$\mathbf{c}_{\mathbf{w}}$ to obtain
	\begin{equation}
		\begin{aligned}
			\|\mathbf{t}_i-\mathbf{c}_{y_i}\|
			&\leq\|\mathcal{F}(\mathbf{x}_i)
			-\mathcal{F}(\mathbf{x}_{\mathbf{w}})\|\\
			&\quad+\left\|\mathcal{F}(\mathbf{x}_{\mathbf{w}})
			-\sum_jw_j\mathcal{F}(\overline{\mathbf{x}}_j)\right\|\\
			&\quad+\left\|\sum_jw_j
			\bigl(\mathcal{F}(\overline{\mathbf{x}}_j)-\mathbf{c}_j\bigr)\right\|
			+\|\mathbf{c}_{\mathbf{w}}-\mathbf{c}_{y_i}\|.
		\end{aligned}
	\end{equation}
	Assumption \ref{ass:regularity} bounds the first term by
	$\ell\|\mathbf{x}_i-\mathbf{x}_{\mathbf{w}}\|$,
	Lemma~\ref{lem:smooth} bounds the second by
	$\frac{\kappa}{2}\sum_jw_j
	\|\overline{\mathbf{x}}_j-\mathbf{x}_{\mathbf{w}}\|^2$, and
	Lemma~\ref{lem:projection_error} bounds the third by
	$\sum_jw_j\varepsilon_{\mathrm{proj}}=\varepsilon_{\mathrm{proj}}$
	which proves the lemma.
\end{proof}

\begin{lemma}
	\label{lem:kl_bound}
	For $\mathbf{a},\mathbf{b}\in\mathbb{R}^{d}$,
	\begin{equation}
		D_{\mathrm{KL}}\!\left(q(\mathbf{a})\,\middle\|\,q(\mathbf{b})\right)
		\leq\frac{1}{2}\|\mathbf{a}-\mathbf{b}\|^2.
		\label{eq:kl_bound}
	\end{equation}
\end{lemma}
\begin{proof}
	Let $\Psi(\mathbf{a})=\log\sum_{j=1}^{C}
	\exp(\mathbf{a}^{\top}\mathbf{c}_j)$. Differentiation gives
	\begin{equation}
		\nabla\Psi(\mathbf{a})=\sum_jq_j(\mathbf{a})\mathbf{c}_j,~~~~
		\nabla^2\Psi(\mathbf{a})
		=\sum_jq_j(\mathbf{a})\mathbf{c}_j\mathbf{c}_j^{\top}
		-\nabla\Psi(\mathbf{a})\nabla\Psi(\mathbf{a})^{\top}.
	\end{equation}
	The Hessian is a covariance matrix. Hence, for any $\mathbf{h}$,
	\begin{equation}
		0\leq\mathbf{h}^{\top}\nabla^2\Psi(\mathbf{a})\mathbf{h}
		\leq\sum_jq_j(\mathbf{a})(\mathbf{h}^{\top}\mathbf{c}_j)^2
		\leq\sum_jq_j(\mathbf{a})\|\mathbf{h}\|^2\|\mathbf{c}_j\|^2
		=\|\mathbf{h}\|^2,
		\label{eq:hessian}
	\end{equation}
	where the last step follows from the $\ell_2$-normalization of the textual
	prototypes, i.e., $\|\mathbf{c}_j\| = 1$ for all $j$.
	Since $\log q_j(\mathbf{a})=
	\mathbf{a}^{\top}\mathbf{c}_j-\Psi(\mathbf{a})$, the KL divergence equals
	\begin{equation}
		\begin{aligned}
			D_{\mathrm{KL}}\!\left(q(\mathbf{a})\,\middle\|\,q(\mathbf{b})\right)
			&=\sum_jq_j(\mathbf{a})
			\left[(\mathbf{a}-\mathbf{b})^{\top}\mathbf{c}_j
			-\Psi(\mathbf{a})+\Psi(\mathbf{b})\right]\\
			&=\Psi(\mathbf{b})-\Psi(\mathbf{a})
			-\nabla\Psi(\mathbf{a})^{\top}(\mathbf{b}-\mathbf{a}).
		\end{aligned}
	\end{equation}
	Set $\mathbf{h}=\mathbf{b}-\mathbf{a}$. Taylor's integral remainder and
	Eq.~\ref{eq:hessian} give
	\begin{equation}
		\begin{aligned}
			D_{\mathrm{KL}}\!\left(q(\mathbf{a})\,\middle\|\,q(\mathbf{b})\right)
			&=\int_0^1(1-t)\mathbf{h}^{\top}
			\nabla^2\Psi(\mathbf{a}+t\mathbf{h})\mathbf{h}\,dt\\
			&\leq\|\mathbf{h}\|^2\int_0^1(1-t)dt
			=\frac{1}{2}\|\mathbf{h}\|^2.
		\end{aligned}
	\end{equation}
\end{proof}

\begin{proof}[\textbf{Proof of Theorem~\ref{thm1:semantic_discrepancy}}]
	Under Assumption~\ref{ass:regularity}, Lemma~\ref{lem:transfer_to_omitted} and Eq.~\ref{eq:nearest_error} imply
	\begin{equation}
		\|\widehat{\mathbf{t}}_i-\mathbf{c}_{y_i}\|
		\leq\|\widehat{\mathbf{t}}_i-\mathbf{t}_i\|
		+\|\mathbf{t}_i-\mathbf{c}_{y_i}\|
		\leq d_i+\varepsilon_{\mathrm{proj}}+b_i.
	\end{equation}
	Applying Lemma~\ref{lem:kl_bound} to
	$(\mathbf{c}_{y_i},\widehat{\mathbf{t}}_i)$ and averaging yields
	\begin{equation}
		\mathcal{E}^{\mathrm{rec}}_o
		\leq\frac{1}{2N_o}\sum_{i\in\mathcal{I}_o}
		\|\widehat{\mathbf{t}}_i-\mathbf{c}_{y_i}\|^2
		\leq\frac{1}{2N_o}\sum_{i\in\mathcal{I}_o}
		(\varepsilon_{\mathrm{proj}}+b_i+d_i)^2.
		\label{eq:sample_bound}
	\end{equation}
	The Euclidean triangle inequality gives
	\begin{equation}
		\begin{aligned}
			\left[\frac{1}{N_o}\sum_{i\in\mathcal{I}_o}
			(\varepsilon_{\mathrm{proj}}+b_i+d_i)^2\right]^{1/2}
			&\leq\varepsilon_{\mathrm{proj}}
			+\left(\frac{1}{N_o}\sum_{i\in\mathcal{I}_o}b_i^2\right)^{1/2}
			+\left(\frac{1}{N_o}\sum_{i\in\mathcal{I}_o}d_i^2\right)^{1/2}\\
			&=\varepsilon_{\mathrm{proj}}+\mathfrak{b}_o+\mathfrak{r}_o.
		\end{aligned}
		\label{eq:euclidean_triangle_inequality}
	\end{equation}
	Squaring Eq.~\ref{eq:euclidean_triangle_inequality} and substituting into
	Eq.~\ref{eq:sample_bound} establishes
	\begin{equation}
		\mathcal{E}^{\mathrm{rec}}_o
		\leq\frac{1}{2}
		\left(\varepsilon_{\mathrm{proj}}+\mathfrak{b}_o+\mathfrak{r}_o\right)^2.
		\label{eq:final_bound}
	\end{equation}
\end{proof}

\textbf{Analysis of Theorem~\ref{thm1:semantic_discrepancy}.} $\varepsilon_{\mathrm{proj}}$ is determined by the quality of cross-modal
semantic projection on the provided classes. It becomes smaller when the
generated textual representations are better aligned with their
corresponding textual prototypes and the pseudo-labeled groups exhibit
more compact visual structures. In LSR, the feature regression objective
$\mathcal{L}_{\mathrm{feat}}$ in Eq.~\ref{eq:feature_regression_loss} directly encourages such semantic alignment,
while balanced high-confidence sample selection helps construct more
compact provided-class groups.

$\mathfrak{b}_o$ characterizes the structural consistency between
the visual and textual spaces which depends on the dataset. It is small when an omitted-class sample and its textual semantics can be well recomposed from the provided-class
structures using consistent convex weights. Therefore, a smooth
cross-modal mapping that preserves the relative structure between the two
modalities facilitates semantic transfer from provided to omitted classes.

Finally, $\mathfrak{r}_o$ depends on how well the retrieval textual
representations cover the generated textual space. In LSR, over-clustering
provides additional semantic anchors, filtering removes unreliable anchors,
and representative retrieval in Eq.~\ref{eq:retrieval} associates the retained anchors with actual generated textual representations. These operations jointly improve the
coverage and reliability of the retrieval textual representations, thereby
reducing its approximation error.

Overall, accurate projection, lower cross-modal recomposition discrepancy and representation approximation error jointly lead to a tighter bound.

\section{Datasets and Evaluations}
\label{app:datasets_and_evaluation}
We conduct experiments on eight image classification datasets from distinct domains, including DTD (DTD.)~\citep{DTD}, FGVC-Aircraft (FA.)~\citep{FGVC-Aircraft}, CUB (CUB.)~\citep{CUB}, Flowers102 (Flw.)~\citep{Flowers102}, OxfordPets (Pets.)~\citep{OxfordPets}, StanfordCars (Cars.)~\citep{StanfordCars}, RESISC45 (RES.)~\citep{RESISC45}, and CIFAR100 (CIFAR.)~\citep{cifar100}. Dataset statistics are presented in Table \ref{tab:datasets}. All datasets adopt the default training-test splits.

\begin{table}[ht]
	\caption{Dataset statistics.}
	\label{tab:datasets}
	\centering
	\begin{tabular*}{1.0\textwidth}{@{}@{\extracolsep{\fill}}ccccccccc@{}}
		\toprule[1.1pt]
		& DTD. & FA. & CUB. & Flw. & Pets. & Cars. & RES. & CIFAR. \\
		\cmidrule(r){1-9}
		Class number & 47 & 100 & 200 & 102 & 37 & 196 & 45 & 100\\
		Training set size & 3760 & 6667 & 5594 & 2040 & 3680 & 8144 & 6300 & 50000\\
		Testing set size & 1880 & 3333 & 5794 & 6149 & 3669 & 8041 & 25200 & 10000\\
		Balanced & \checkmark & \checkmark & \ding{55} & \ding{55} & \ding{55} & \ding{55} & \checkmark & \checkmark \\
		\bottomrule[1.1pt]
	\end{tabular*}
\end{table}

In our experiments, we evaluate methods with three metrics, including omitted-class accuracy, omitted-group recall, and all-class accuracy. Let $\mathcal{D}_{\mathrm{test}}=\{(x_i,y_i)\}_{i=1}^{N}$ denote the test set, where predictions are made over the set of all class names
$\mathcal{C}_t=\mathcal{C}_p\cup\mathcal{C}_o$.
Let $\hat{y}_i$ be the predicted label of $x_i$, and define
$\mathcal{I}_o=\{i\mid y_i\in\mathcal{C}_o\}$, $N_o=\left| \mathcal{I} _o \right|$.

\textbf{Omitted-class accuracy} (OCA) measures the classification accuracy of omitted-class samples:
\begin{equation}
	\mathrm{OCA}
	=\frac{1}{N_o}\sum_{i\in\mathcal{I}_o}\mathbf{1}\!\left(\hat{y}_i=y_i\right)\times 100\%.
\end{equation}

\textbf{Omitted-group recall} (OGR) treats all omitted classes as a single group and measures the proportion of omitted-class samples that are predicted as any omitted class:
\begin{equation}
	\mathrm{OGR}
	=
	\frac{1}{N_o}
	\sum_{i\in\mathcal{I}_o}
	\mathbf{1}\!\left(\hat{y}_i\in\mathcal{C}_o\right)
	\times 100\%.
\end{equation}
Unlike OCA, OGR does not require the exact omitted-class label to be correctly predicted. Instead, it reflects how well the model avoids misclassifying omitted-class samples as provided classes.

\textbf{All-class accuracy} (ACC) measures the overall classification accuracy across both provided and omitted classes:
\begin{equation}
	\mathrm{ACC}
	=
	\frac{1}{N}
	\sum_{i=1}^{N}
	\mathbf{1}\!\left(\hat{y}_i=y_i\right)
	\times 100\%.
\end{equation}
$\mathbf{1}(\cdot)$ denotes the indicator function, which equals $1$ when the condition is satisfied and $0$ otherwise.

\section{Comparison Methods}
\label{app:comparison_methods}
We compare our method with ten SOTA methods, including CLIP~\citep{CLIP}, Few Pseudo-Labels (FPL)~\citep{FPL_GRIP}, Grow and Refine Iteratively Pseudo-labels (GRIP)~\citep{FPL_GRIP}, Candidate Pseudolabel Learning (CPL)~\citep{CPL}, Universal Entropy Optimization (UEO)~\citep{UEO}, Augmenting Discriminative Richness (AiR)~\citep{AiR}, Concept-Adaptive Pseudo-labeling (CAP)~\citep{CAP}, Maximum Concept Matching (MCM)~\citep{MCM}, M-Tuning~\citep{M-Tuning}, and BGAT~\citep{BGAT}. Among these methods, MCM, M-Tuning, and BGAT were originally designed for open-set scenarios, and we adapt them to the VI-UVA setting. The details of each method are presented as follows:
\begin{itemize}[leftmargin=*]
	\item \textbf{CLIP: } CLIP exhibits zero-shot capability for classification.
	\item \textbf{FPL: } FPL filters the top-$k$ high-confidence pseudo-labeled samples for each class to construct a balanced pseudo-label dataset for model training.
	\item \textbf{GRIP: } GRIP predefines the number of iterations. Within each iteration, it regenerates pseudo-labels and filters the top-$k$ high-confidence pseudo-labeled samples for each class to construct a balanced pseudo-label dataset, where $k$ increases linearly as iterations proceed.
	\item \textbf{CPL: } CPL uses intra- and inter-instance selected candidate pseudo-labels instead of hard pseudo-labels, effectively reducing label noise and class imbalance for better adaptation.
	\item \textbf{UEO: } UEO is a universal entropy optimization method for realistic unsupervised CLIP fine-tuning. It uses sample confidence to minimize entropy for in-distribution samples and maximize marginal entropy for OOD samples.
	\item \textbf{AiR: } AiR utilizes LoRA‑fine‑tuned diffusion models to generate class‑representative synthetic images and builds an auxiliary classifier for high‑quality pseudo‑labels.
	\item \textbf{CAP: } CAP proposes concept alignment to fix semantic-visual misalignment, and a confusion-aware calibrated margin to improve discrimination between similar classes. It also adopts dual independent adapters to separately learn balanced pseudo-labels and high-confidence unlabeled samples.
	\item \textbf{MCM: } MCM is adapted to VI-UVA by computing concept-matching scores over the adaptation vocabulary and using these scores to select samples for prompt tuning.
	\item \textbf{M-Tuning: } M-Tuning is adapted to VI-UVA by augmenting the adaptation vocabulary with auxiliary open words and incorporating them into prompt tuning. These additional words expand the set of textual concepts considered during optimization, helping mitigate the model's bias toward the provided classes.
	\item \textbf{BGAT: } We adapt BGAT to VI-UVA by applying a Box–Cox transformation and a two-component Gaussian mixture model to sample scores computed over the adaptation vocabulary, yielding a data-adaptive threshold. This threshold filters out potential omitted-class samples.
\end{itemize}

\section{Implementation}
\label{app:implementation}

\subsection{Pseudo-Labeled Dataset Generation}
\label{app:pseudo_labeled_dataset_generation}
We first exploit the zero-shot predictions of the pre-trained VLM to assign each $x_i$ in $\mathcal{D}_{\mathrm{U}}$ to the most likely provided class. Let $\hat{y}_i\in\mathcal{C}_p$ denote the predicted class of $x_i$. For each provided class $c\in\mathcal{C}_p$, we define the set of assigned sample indices and its size as
\begin{equation}
	\mathcal{S}_c=\{i\mid\hat{y}_i=c\},~~
	n_c=|\mathcal{S}_c|.
\end{equation}
Then, the provided classes are sorted in ascending order of their sample counts. Let $\pi$ denote the resulting ordering, such that
\begin{equation}
	n_{\pi(1)}\leq n_{\pi(2)}\leq\cdots\leq n_{\pi(|\mathcal{C}_p|)}.
\end{equation}
The first $\lfloor \varepsilon \left| \mathcal{C} _p \right| \rfloor $ classes constitute the tail-class subset
$\mathcal{C}_{\mathrm{tail}}=\{\pi(1),\ldots,\pi(\left|\mathcal{C}_{\mathrm{tail}}\right|)\}$, where $\varepsilon \in \left( 0,1 \right) $ denotes the ratio of tail classes.
To mitigate class imbalance, we use their average sample count to determine a per-class upper bound:
\begin{equation}
	U=\left\lfloor\frac{1}{\left|\mathcal{C}_{\mathrm{tail}}\right|}
	\sum_{j=1}^{\left|\mathcal{C}_{\mathrm{tail}}\right|}n_{\pi(j)}\right\rfloor.
\end{equation}
For each provided class, we compute the confidence score of each $x_i$ in $\mathcal{D}_{\mathrm{U}}$ for that class and select the $U$ highest-confidence images. Each selected image is paired with the corresponding class as its pseudo-label, and these pairs are combined into a balanced pseudo-labeled dataset $\mathcal{D}_{\mathrm{P}}$ for training the cross-modal semantic projector. To construct $\mathcal{D}_{\mathrm{R}}$, we follow the same procedure, using $\mathcal V_R$ in place of $\mathcal C_p$.

\subsection{Filtering Procedure of Unreliable Anchors}
\label{app:filtering_procedure_of_unreliable_centers}
Over-clustering may produce unreliable semantic anchors that capture outliers or noise. We evaluate the reliability of each candidate anchor by the number of unlabeled images assigned to it. Since the semantic anchors lie in the textual embedding space shared with the visual embedding space, we calculate the cosine similarity between each visual embedding $\mathbf{x}_i$ and every semantic anchor $\boldsymbol{\mu}_k$ and assign $\mathbf{x}_i$ to its most similar anchor:
\begin{equation}
	a_i=\underset{k\in \left\{1,...,K\right\}}{\mathrm{arg}\max}\frac{\boldsymbol{\mu }_{k}^{\top}\mathbf{x}_i}{\left\| \boldsymbol{\mu }_k \right\| _2\left\| \mathbf{x}_i \right\| _2}.
\end{equation}
where $K$ is the number of candidate anchors. The support count of the anchor $\boldsymbol{\mu}_k$ is defined as:
\begin{equation}
	S_k=\sum_{i=1}^{N}\mathbf{1}(a_i=k),
\end{equation}
where $\mathbf{1}(\cdot)$ is the indicator function, which equals $1$ when the specified condition is satisfied and $0$ otherwise. $N$ represents the number of images in $\mathcal{D}_\mathrm{U}$. We retain anchors whose support counts exceed the threshold $\lfloor {\delta N/\left| \mathcal{C} _p \right|} \rfloor $, where $\delta$ controls the minimum support count and is fixed at 0.05 in our implementation. This procedure filters out anchors supported by very few images, helping to suppress noise while retaining the dominant semantic structures in the unlabeled data.

\subsection{Implementation Details}
\label{app:implementation_details}
We adopt CLIP with a ViT-B/32 image encoder and set the number of learnable context tokens $S$ to 16. The parameter $\varepsilon$ for constructing balanced pseudo-labeled datasets takes values in  $\left\{0.2, 0.3, 0.4\right\}$. The batch size is set to 32.
We train CMSP for 50 epochs using AdamW with a learning rate of $10^{-4}$ and a weight decay of $10^{-4}$. The length of the pseudo-token sequence $L$ is set to 7, and the loss coefficient $\alpha$ is fixed at 0.1. The parameter $\delta$ for filtering unreliable anchors is fixed at 0.05.
For textual prompt tuning, we optimize the shared context tokens using SGD for a maximum of 50 epochs, with an initial learning rate of 0.01, a cosine annealing schedule, and a weight decay of 0.05. The temperature parameter $\tau$ is fixed at 1.0, and the loss coefficient $\lambda$ takes values in $\left\{0.2, 1.0, 2.0, 5.0\right\}$. Sensitivity analyses of the parameters are presented in Sec~\ref{sec:more_analysis} and Appendix~\ref{app:sensitivity_analysis}.

\section{More Experiment Results and Analysis}
\label{app:more_experiment_results_and_analysis}

\subsection{More Main Results}
\label{app:additional_main_results}
Table~\ref{tab:main_results_gamma_0.4} compares the all-class accuracy of different methods at omission ratio of $\gamma=0.4$. All adaptation methods except LSR perform worse than the CLIP baseline on average, while LSR outperforms the CLIP baseline by an average of \textbf{2.34} percentage points. These results demonstrate the effectiveness of LSR at a larger omission ratio.
 
\begin{table*}[ht]
	\centering
	\caption{Comparison results of all-class accuracy (\%) at $\gamma=0.4$. Values are expressed as $\mathrm{mean}_{\mathrm{std}}$. The highest accuracies are bold.}
	\vspace{-0.2cm}
	\label{tab:main_results_gamma_0.4}
	\small
	\renewcommand{\arraystretch}{1.1}
	\resizebox{1\linewidth}{!}{
		\begin{tabular}{l|ccccccccc}
			\toprule[1.1pt]
			Methods & Flw. & FA. & CUB. & DTD. & Pets. & Cars. & RES. & CIFAR. & Avg. \\
			\cmidrule(lr){1-1} \cmidrule(lr){2-10}
			CLIP & $63.63_{0.00}$ & $17.64_{0.00}$ & $51.77_{0.00}$ & $43.19_{0.00}$ & $84.51_{0.00}$ & $58.93_{0.00}$ & $54.48_{0.00}$ & $64.19_{0.00}$ & $54.79_{0.00}$ \\
			FPL & $62.49_{0.24}$ & $15.61_{0.94}$ & $49.41_{0.38}$ & $43.99_{0.46}$ & $88.01_{1.22}$ & $53.83_{0.47}$ & $58.28_{0.74}$ & $64.93_{0.68}$ & $54.57_{0.64}$ \\
			GRIP & $55.36_{0.53}$ & $12.67_{1.08}$ & $39.51_{0.86}$ & $42.25_{1.58}$ & $75.39_{2.98}$ & $45.02_{0.97}$ & $57.35_{2.30}$ & $61.31_{0.35}$ & $48.61_{1.33}$ \\
			CPL & $54.74_{0.94}$ & $14.20_{0.35}$ & $44.24_{0.97}$ & $41.12_{1.99}$ & $79.54_{1.02}$ & $43.81_{0.46}$ & $55.98_{0.71}$ & $58.54_{0.88}$ & $49.02_{0.92}$ \\
			CAP & $52.78_{3.44}$ & $13.64_{1.67}$ & $44.66_{0.91}$ & $40.66_{1.58}$ & $73.21_{0.19}$ & $50.27_{0.32}$ & $60.94_{0.58}$ & $61.70_{0.04}$ & $49.73_{1.09}$ \\
			AiR & $55.69_{1.88}$ & $12.41_{0.65}$ & $37.94_{0.12}$ & $42.06_{1.18}$ & $76.01_{2.08}$ & $45.14_{0.50}$ & $56.09_{0.76}$ & $61.09_{0.50}$ & $48.30_{0.96}$ \\
			MCM & $59.32_{0.37}$ & $14.29_{2.09}$ & $44.76_{0.73}$ & $39.49_{1.22}$ & $80.20_{2.98}$ & $50.11_{1.21}$ & $49.73_{0.35}$ & $63.55_{0.38}$ & $50.18_{1.17}$ \\
			UEO & $58.35_{1.48}$ & $8.22_{1.89}$ & $42.57_{0.77}$ & $40.43_{0.87}$ & $80.94_{1.03}$ & $49.04_{2.02}$ & $57.38_{0.59}$ & $45.56_{2.58}$ & $47.81_{1.40}$ \\
			M-Tuning & $61.55_{0.28}$ & $15.98_{0.96}$ & $49.36_{0.33}$ & $43.01_{0.27}$ & $\textbf{88.01}_{0.31}$ & $53.39_{0.77}$ & $58.31_{1.60}$ & $64.48_{0.40}$ & $54.26_{0.62}$ \\
			BGAT & $58.06_{0.96}$ & $16.69_{0.27}$ & $46.57_{0.66}$ & $39.18_{0.76}$ & $82.80_{0.50}$ & $55.15_{0.72}$ & $51.95_{0.20}$ & $64.54_{0.13}$ & $51.87_{0.53}$ \\
			\cellcolor{Gray}LSR (Ours) & \cellcolor{Gray}$\textbf{64.36}_{0.73}$ & \cellcolor{Gray}$\textbf{18.55}_{0.05}$ & \cellcolor{Gray}$\textbf{52.43}_{0.18}$ &
			\cellcolor{Gray}$\textbf{45.59}_{0.52}$ & \cellcolor{Gray}${86.33}_{0.55}$ & \cellcolor{Gray}$\textbf{59.06}_{0.48}$ &
			\cellcolor{Gray}$\textbf{64.52}_{0.39}$ & \cellcolor{Gray}$\textbf{66.18}_{0.25}$ & \cellcolor{Gray}$\textbf{57.13}_{0.39}$ \\
			\bottomrule[1.1pt]
	\end{tabular}}
\end{table*}

\subsection{More Ablation Studies}
\label{app:more_ablation_studies}
As shown in Table \ref{tab:more_ablation_study}, incorporating $\mathcal{L}_{\mathrm{R}}$ improves omitted-class accuracy across all six datasets, with an average gain of \textbf{3.39} percentage points. By introducing supervision from the reconstructed latent label space, $\mathcal{L}_{\mathrm{R}}$ encourages prompts to retain semantic distinctions beyond the adaptation vocabulary.
\begin{table*}[ht] 
	\centering 
	\caption{Comparison of omitted-class accuracy (\%) for ablation variants w and w/o $\mathcal{L}_\mathrm{R}$ at $\gamma=0.4$. Values are expressed as $\mathrm{mean}_{\mathrm{std}}$. The highest accuracies are bold.} 
	\vspace{-0.2cm}
	\label{tab:more_ablation_study} 
	\small
	\renewcommand{\arraystretch}{1.1} 
	\resizebox{0.8\linewidth}{!}{ 
		\begin{tabular}{c|ccccccc} 
			\toprule[1.1pt] 
			$\mathcal{L}_\mathrm{R}$ & Flw. & FA. & CUB. & Pets. & Cars. & RES. & Avg. \\  
			\midrule 
			\textbf{\ding{55}} & $66.92_{1.14}$ & $18.30_{1.75}$ & $39.18_{2.11}$ & $85.10_{1.31}$ & $57.81_{0.49}$ & $57.96_{0.00}$ & $54.21_{1.25}$  \\  
			\textbf{\checkmark} & $\textbf{72.00}_{0.69}$ & $\textbf{20.33}_{0.12}$ & $\textbf{44.16}_{0.24}$ & $\textbf{89.31}_{0.47}$ & $\textbf{59.54}_{0.30}$ & $\textbf{60.29}_{0.41}$ & $\textbf{57.60}_{0.37}$ \\ 
			\bottomrule[1.1pt] 
	\end{tabular}} 
\end{table*}

\subsection{Comparison with WordNet-Based Vocabulary Expansion}
\label{app:wordnets}
To investigate an alternative way to mitigate vocabulary overfitting, we expand the adaptation vocabulary with additional words retrieved from WordNet. For each unlabeled image, we seek words that better match its visual embedding than the class names in the adaptation vocabulary. These words supplement the adaptation vocabulary to provide supervision beyond the provided classes.

Table~\ref{tab:WordNets} shows that LSR consistently outperforms this strategy across all eight datasets and three omission ratios. At $\gamma=0.2$, $0.3$, and $0.4$, LSR yields gains of \textbf{4.08}, \textbf{4.53}, and \textbf{4.68} percentage points in all-class accuracy.
These results suggest that retrieving semantically similar words alone is insufficient to construct effective supervision for adaptation. Although additional words may better match individual images, higher image-text similarity does not necessarily imply that they capture the distinctions among classes. Retrieved words may describe shared attributes or broad concepts, providing limited class-discriminative supervision. In contrast, LSR constructs semantic anchors by clustering representations from the unlabeled data, allowing the reconstructed label space to reflect its latent semantic structures.

\begin{table*}[ht] 
	\centering 
	\caption{Comparison of all-class accuracy (\%) between WordNet and LSR at different $\gamma$.Values are expressed as $\mathrm{mean}_{\mathrm{std}}$. The highest accuracies are bold.} 
	\vspace{-0.2cm}
	\label{tab:WordNets} 
	\small
	\renewcommand{\arraystretch}{1.1} 
	\resizebox{1.0\linewidth}{!}{ 
		\begin{tabular}{cc|ccccccccc} 
			\toprule[1.1pt] 
			$\gamma$ & Methods & Flw. & FA. & CUB. & DTD. & Pets. & Cars. & RES. & CIFAR. & Avg. \\  
			\midrule 
			\multirow{2}{*}{0.2} & WordNet & $62.77_{0.42}$ & $16.67_{0.23}$ & $51.04_{0.73}$ & $42.07_{0.54}$ & $85.79_{2.07}$ & $57.22_{0.53}$ & $57.12_{1.13}$ & $62.86_{0.45}$ & $54.44_{0.76}$  \\  
			& LSR (Ours) & \cellcolor{Gray}$\textbf{66.89}_{0.60}$ & \cellcolor{Gray}$\textbf{19.34}_{0.41}$ & \cellcolor{Gray}$\textbf{53.27}_{0.22}$ &
			\cellcolor{Gray}$\textbf{47.02}_{0.54}$ & \cellcolor{Gray}$\textbf{88.49}_{0.60}$ & \cellcolor{Gray}$\textbf{60.15}_{0.34}$ &
			\cellcolor{Gray}$\textbf{65.82}_{0.35}$ & \cellcolor{Gray}$\textbf{67.14}_{0.53}$ & \cellcolor{Gray}$\textbf{58.52}_{0.45}$ \\
			\midrule 
			\midrule 
			\multirow{2}{*}{0.3} & WordNet & $55.34_{0.76}$ & $16.61_{0.20}$ & $50.56_{0.24}$ & $40.34_{0.87}$ & $86.85_{0.69}$ & $55.94_{0.73}$ & $56.79_{0.73}$ & $63.08_{0.36}$ & $53.19_{0.57}$  \\  
			& LSR (Ours) & \cellcolor{Gray}$\textbf{65.50}_{0.42}$ & \cellcolor{Gray}$\textbf{18.92}_{0.18}$ & \cellcolor{Gray}$\textbf{53.00}_{0.50}$ &
			\cellcolor{Gray}$\textbf{45.32}_{0.52}$ & \cellcolor{Gray}$\textbf{87.53}_{0.33}$ & \cellcolor{Gray}$\textbf{59.68}_{0.29}$ &
			\cellcolor{Gray}$\textbf{65.26}_{0.41}$ & \cellcolor{Gray}$\textbf{66.53}_{0.14}$ & \cellcolor{Gray}$\textbf{57.72}_{0.35}$ \\ 
			\midrule 
			\midrule 
			\multirow{2}{*}{0.4} & WordNet & $56.64_{0.39}$ & $15.31_{0.45}$ & $49.18_{1.53}$ & $41.01_{0.28}$ & $84.11_{1.23}$ & $54.29_{0.78}$ & $56.31_{0.32}$ & $62.71_{1.36}$ & $52.45_{0.79}$  \\ 
			& LSR (Ours) & \cellcolor{Gray}$\textbf{64.36}_{0.73}$ & \cellcolor{Gray}$\textbf{18.55}_{0.05}$ & \cellcolor{Gray}$\textbf{52.43}_{0.18}$ &
			\cellcolor{Gray}$\textbf{45.59}_{0.52}$ & \cellcolor{Gray}$\textbf{86.33}_{0.55}$ & \cellcolor{Gray}$\textbf{59.06}_{0.48}$ &
			\cellcolor{Gray}$\textbf{64.52}_{0.39}$ & \cellcolor{Gray}$\textbf{66.18}_{0.25}$ & \cellcolor{Gray}$\textbf{57.13}_{0.39}$ \\
			\bottomrule[1.1pt] 
		\end{tabular}} 
\end{table*}

\subsection{Sensitivity Analysis}
\label{app:sensitivity_analysis}
We extend the sensitivity analysis of $\lambda$ and $\varepsilon$ to additional datasets. As shown in Fig.~\ref{fig:more_sensitivity_analysis}, LSR maintains relatively stable all-class accuracy as either parameter varies within the tested range. These results further support the robustness of LSR to the choice of these two hyperparameters.

\begin{figure*}[ht]
	\centering
	\begin{subfigure}[t]{0.24\textwidth}
		\centering
		\includegraphics[width=\linewidth]{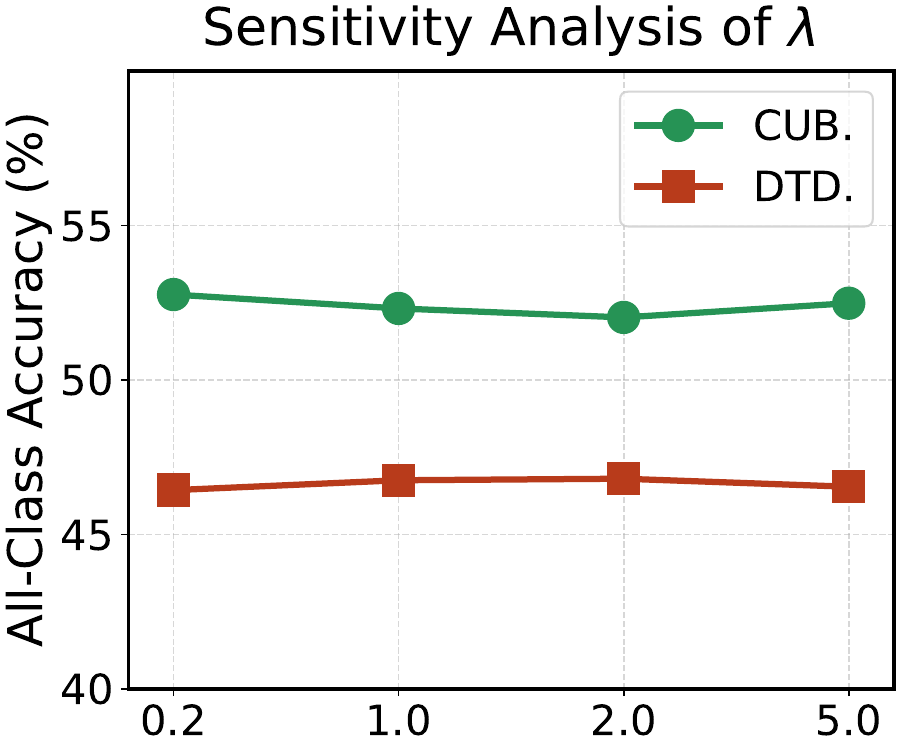}
		\caption{}
		\label{fig:sensitivity_analysis_lambda_cub_dtd}
	\end{subfigure}
	\hfill
	\begin{subfigure}[t]{0.24\textwidth}
		\centering
		\includegraphics[width=\linewidth]{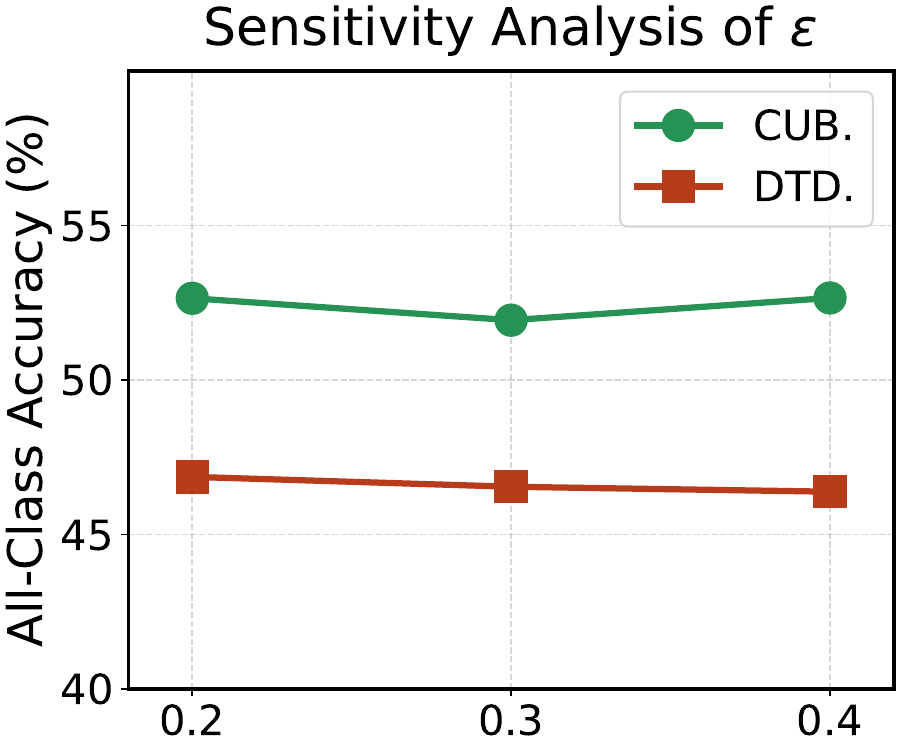}
		\caption{}
		\label{fig:sensitivity_analysis_epsilon_cub_dtd}
	\end{subfigure}
	\hfill
	\begin{subfigure}[t]{0.24\textwidth}
		\centering
		\includegraphics[width=\linewidth]{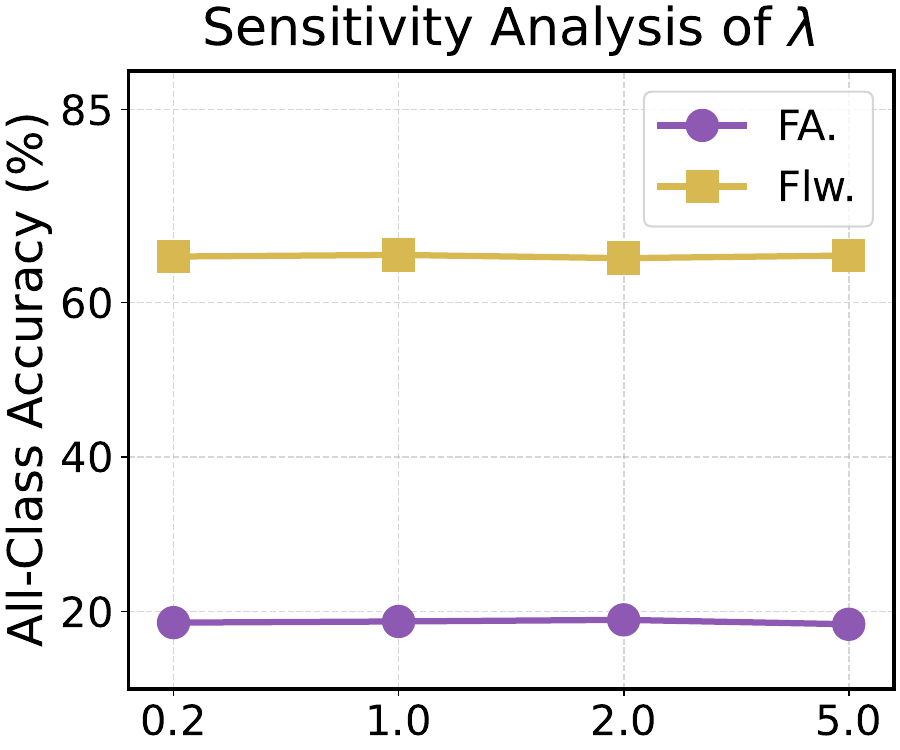}
		\caption{}
		\label{fig:sensitivity_analysis_lambda_fa_flw}
	\end{subfigure}
	\hfill
	\begin{subfigure}[t]{0.24\textwidth}
		\centering
		\includegraphics[width=\linewidth]{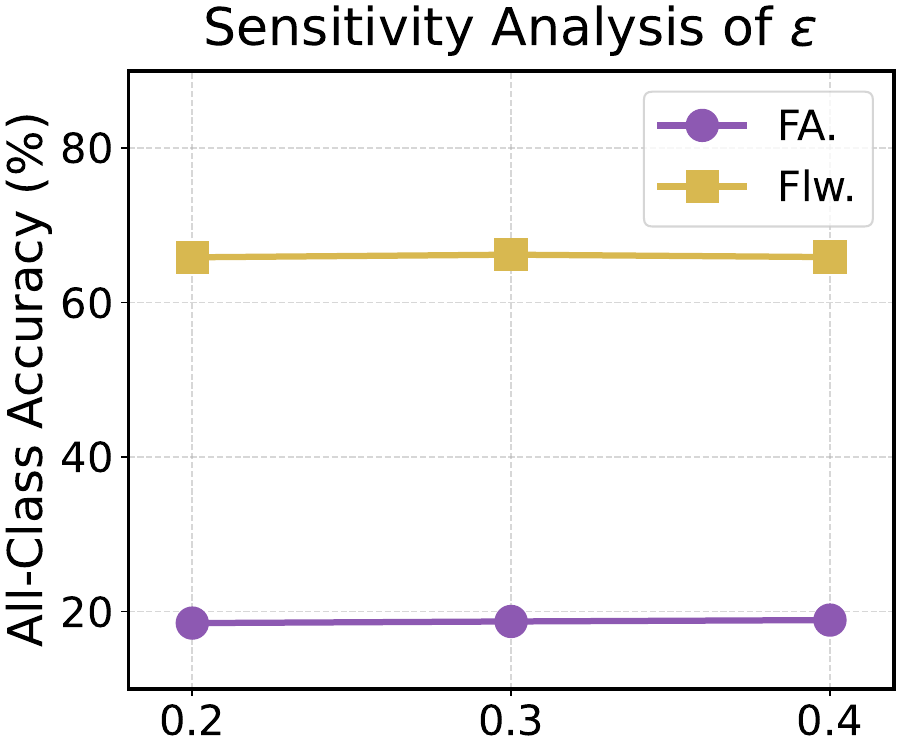}
		\caption{}
		\label{fig:sensitivity_analysis_epsilon_fa_flw}
	\end{subfigure}
	\hfill
	\caption{(a) Sensitivity analysis of $\lambda$ on dataset CUB. and DTD. (b) Sensitivity analysis of $\varepsilon$ on dataset CUB. and DTD. (c) Sensitivity analysis of $\lambda$ on dataset FA. and Flw. (d) Sensitivity analysis of $\varepsilon$ on dataset FA. and Flw.}
	\label{fig:more_sensitivity_analysis}
\end{figure*}

We also perform sensitivity analysis of $\alpha$ (in Eq.~\ref{eq:CMSP_loss}) on different datasets. As shown in Fig.~\ref{fig:more_sensitivity_analysis_alpha}, LSR maintains relatively stable as $\alpha$ varies within the tested range. These results support the robustness of LSR to the choice of $\alpha$. Fixing $\alpha$ at 0.1 is a simple and effective strategy.

\begin{figure*}[ht]
	\centering
	\begin{subfigure}[t]{0.3\textwidth}
		\centering
		\includegraphics[width=\linewidth]{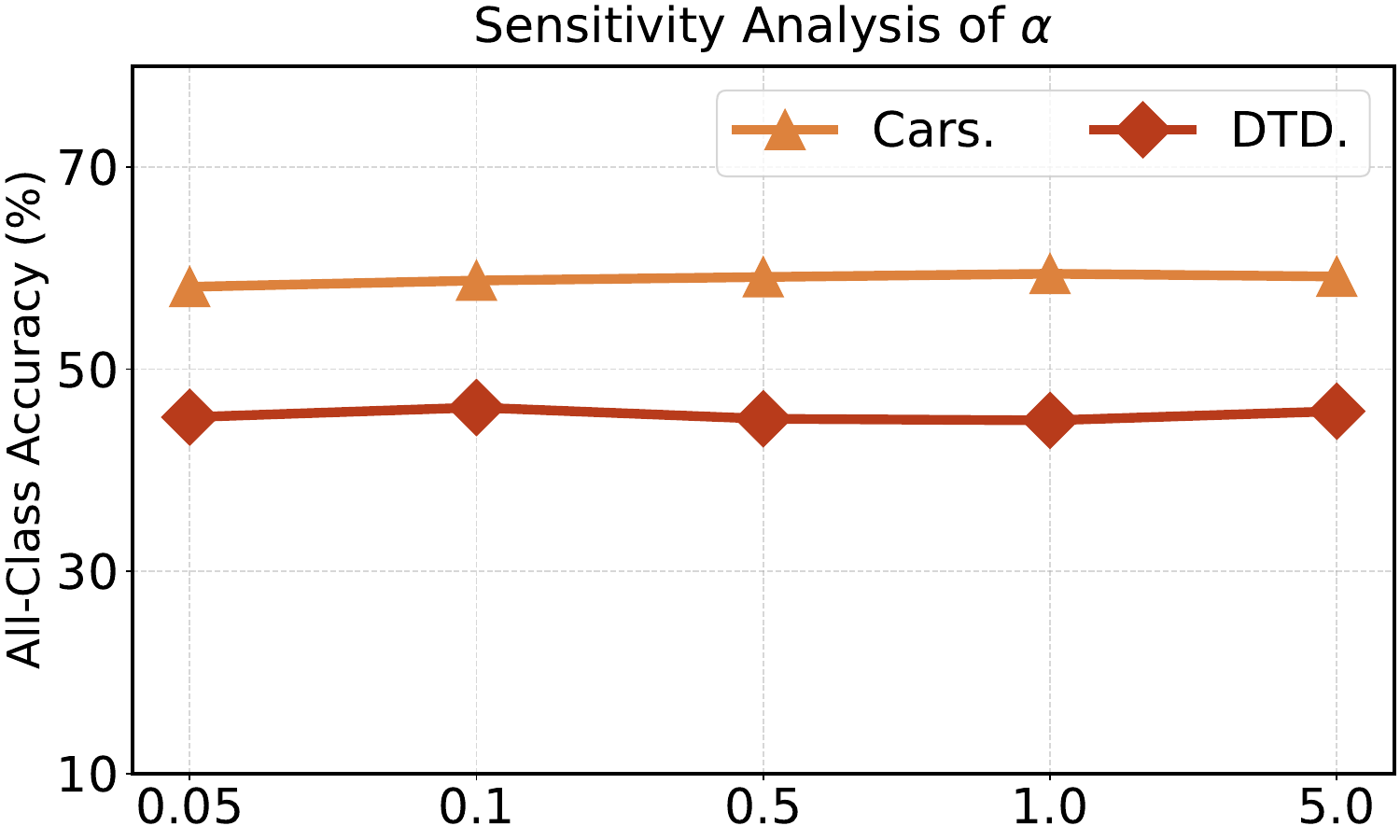}
		\vspace{-0.5cm}
		\caption{}
		\label{fig:sensitivity_analysis_alpha_Cars_DTD}
	\end{subfigure}
	\hspace{0.08\textwidth}%
	\begin{subfigure}[t]{0.3\textwidth}
		\centering
		\includegraphics[width=\linewidth]{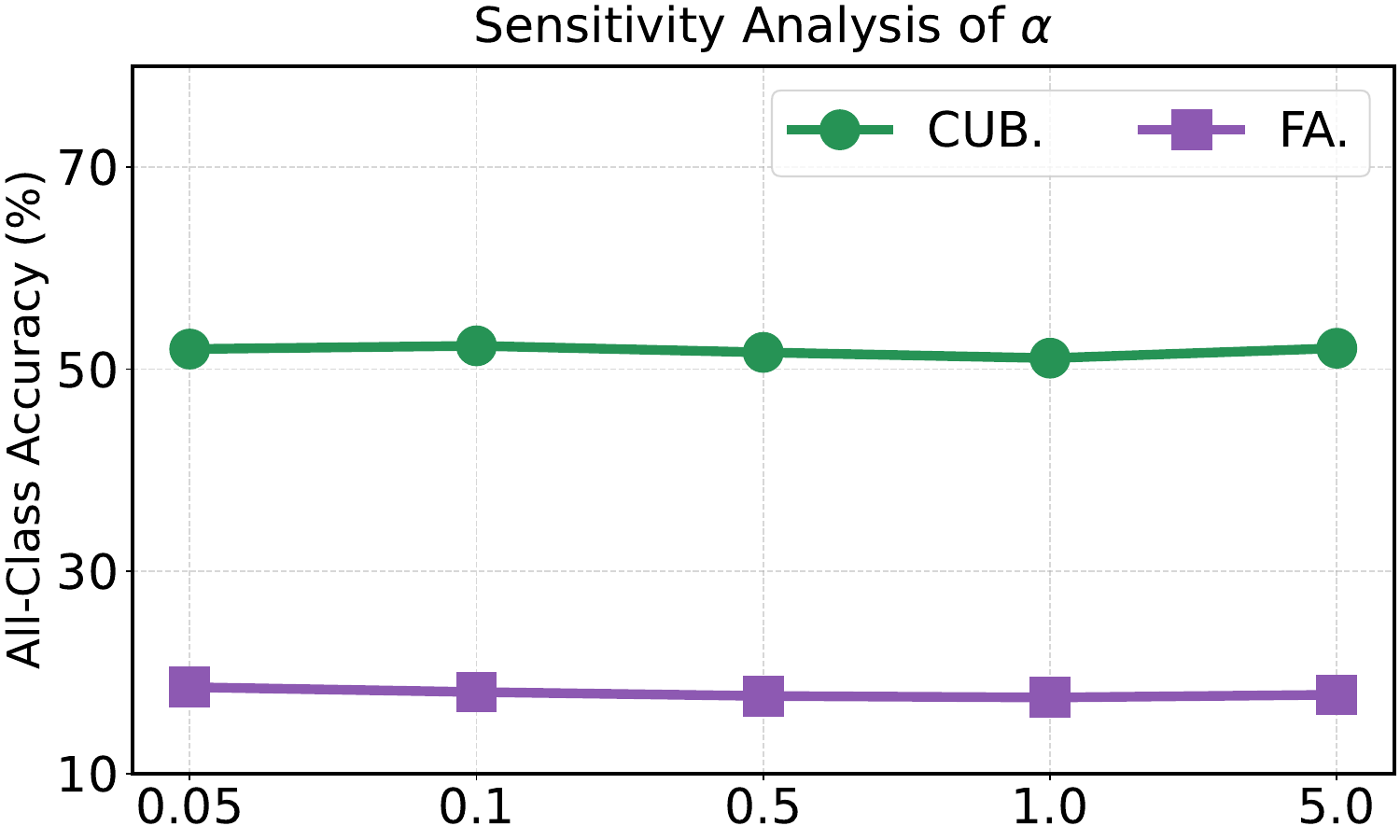}
		\vspace{-0.5cm}
		\caption{}
		\label{fig:sensitivity_analysis_epsilon_CUB_FA}
	\end{subfigure}
	\vspace{-0.1cm}
	\caption{(a) Sensitivity analysis of $\alpha$ on datasets Cars. and DTD. (b) Sensitivity analysis of $\alpha$ on datasets CUB. and FA.}
	\label{fig:more_sensitivity_analysis_alpha}
\end{figure*}

The cluster ratio denotes the ratio of the number of cluster centers (in Sec.~\ref{sec:LLSR}) to the number of provided classes and is set to 2.0 by default. Table~\ref{tab:sensitivity_analysis_cluster} suggests that low cluster ratios generally yield poorer performance. The limited number of centers may cause omitted-class samples to share semantic anchors with provided-class samples, obscuring their semantic distinctions. Over-clustering provides additional centers, which can help separate these samples and capture semantic structures beyond the adaptation vocabulary. However, further increases in the number of centers do not consistently improve performance.

Table~13 examines the sensitivity of LSR to the anchor-filtering parameter $\delta$ (in Sec.~\ref{app:filtering_procedure_of_unreliable_centers}). Setting $\delta=0.0$ disables filtering and yields slightly lower accuracy, suggesting that retaining unreliable anchors can weaken reconstruction-based supervision. Conversely, increasing $\delta$ to 0.5 also reduces performance, as overly aggressive filtering may remove useful anchors. Therefore, fixing $\delta$ to a small value of 0.05 is sufficient to achieve competitive performance.

\begin{figure}[ht]
	\begin{minipage}[ht]{0.53\textwidth}
		\centering
		\captionof{table}{Sensitivity to the cluster ratios.}
		\label{tab:sensitivity_analysis_cluster}
		\renewcommand{\arraystretch}{1.0}
		\resizebox{1.0\linewidth}{!}{ 
			\begin{tabular}{c|ccccc}
				\toprule[1.1pt]
				Dataset & 0.1 & 0.2 & 1.0 & 2.0 & 3.0 \\
				\midrule
				DTD. & 44.09 & 45.37 & 45.31 & 46.86 & 45.53 \\
				Flw. & 63.21 & 63.66 & 63.35 & 64.53 & 65.19 \\
				Cars. & 57.76 & 58.58 & 58.50 & 59.24 & 57.97 \\
				RES. & 62.02 & 63.24 & 63.55 & 64.52 & 63.77 \\
				\bottomrule[1.1pt]
			\end{tabular}}
	\end{minipage}
	\hfill
	\begin{minipage}[ht]{0.45\textwidth}
	\centering
	\captionof{table}{Sensitivity analysis of $\delta$.}
	\label{tab:sensitivity_analysis_retain}
	\renewcommand{\arraystretch}{1.0}
	\resizebox{1.0\linewidth}{!}{ 
		\begin{tabular}{c|cccc}
			\toprule[1.1pt]
			Dataset & 0.0 & 0.05 & 0.1 & 0.5 \\
			\midrule
			DTD. & 45.05 & 45.59 & 45.53 & 45.31 \\
			CUB. & 64.12 & 64.53 & 64.92 & 64.30 \\
			Cars. & 63.45 & 64.52 & 64.25 & 63.51 \\
		    FA. & 17.88 & 18.55 & 18.39 & 16.23 \\
			\bottomrule[1.1pt]
		\end{tabular}}
\end{minipage}
\end{figure}

\section{Compute Resources}
\label{app:compute_resources}
\begin{itemize}[leftmargin=*]
	\item CPU: AMD EPYC 7642 48-Core Processor × 2
	
	\item GPU: NVIDIA GeForce RTX 4090 24G × 1
	
	\item MEM: 500G
	
	\item Maximum total computing time per dataset: training + testing $\approx$ 30min.
\end{itemize}

\end{document}